\documentclass{article}
\usepackage{url}            
\usepackage{booktabs}       
\usepackage{amsfonts}       
\usepackage{nicefrac}       
\usepackage{microtype}      
\usepackage{amsthm}
\usepackage{multirow}

\usepackage[accepted]{icml2020}
\usepackage{amsmath}
\usepackage{amssymb}
\usepackage{dsfont}
\usepackage{graphicx}
\usepackage{glossaries}
\usepackage{subcaption}

\usepackage[usenames,dvipsnames]{xcolor}
\definecolor{shadecolor}{gray}{0.9}

\newcommand{\red}[1]{\textcolor{BrickRed}{#1}}
\newcommand{\orange}[1]{\textcolor{BurntOrange}{#1}}
\newcommand{\green}[1]{\textcolor{OliveGreen}{#1}}
\newcommand{\blue}[1]{\textcolor{MidnightBlue}{#1}}
\newcommand{\sky}[1]{\textcolor{SkyBlue}{#1}}

\usepackage{natbib}
\usepackage[colorlinks,linktoc=all,pdflang=en-US]{hyperref}
\usepackage[all]{hypcap}
 \hypersetup{citecolor=MidnightBlue}
 \hypersetup{linkcolor=MidnightBlue}
 \hypersetup{urlcolor=MidnightBlue}

\usepackage[nameinlink]{cleveref}
\creflabelformat{equation}{#2\textup{#1}#3}  

\usepackage{framed}

\newcommand{\note}[1]{}  

\usepackage{ragged2e}
\DeclareRobustCommand{\sidenote}[1]{} 

\def\rT{{\rm T}}
\def\mI{\mathrm{I}}

\newcommand*\E[1]{\mathbb{E}\left[#1\right]}
\newcommand*\Ep[2]{\mathbb{E}_{#1}\left[#2\right]}
\newcommand*\lrw[1]{\left\langle#1\right\rangle}

\newcommand*\lrp[1]{\left(#1\right)}

\newtheorem{theorem}{Theorem}
\newtheorem{lemma}{Lemma}

\newlength{\figwidth}
\hypersetup{
  pdflang={en-US}, 
} \newacronym{ace}{ACE}{Adaptive Calibration Error}
\newacronym{aucpr}{AUC-PR}{area under the precision-recall curve}
\newacronym{aucroc}{AUC-ROC}{area under the receiver operating characteristic curve}
\newacronym{ccs}{CCS}{Clinical Classifications Software}
\newacronym{ece}{ECE}{Expected Calibration Error}
\newacronym{elbo}{ELBO}{evidence lower bound}
\newacronym{icu}{ICU}{intensive care unit}
\newacronym{kl}{KL}{Kullback-Leibler}
\newacronym{lstm}{LSTM}{long short-term memory}
\newacronym{mimic}{MIMIC-III}{Medical Information Mart for Intensive Care}
\newacronym{nn}{NN}{neural network}
\newacronym{ppv}{PPV}{positive predictive value}
\newacronym{pr}{PR}{precision-recall}
\newacronym{rnn}{RNN}{recurrent neural network}
\newacronym{vi}{VI}{variational inference} 
\begin{document}
\setlength{\footskip}{3em}
\setlength{\abovedisplayskip}{3pt}
\setlength{\belowdisplayskip}{3pt}

\icmltitlerunning{Efficient and Scalable Bayesian Neural Nets with Rank-1 Factors}

\twocolumn[
\icmltitle{Efficient and Scalable Bayesian Neural Nets with Rank-1 Factors}

\icmlsetsymbol{equal}{*}
\icmlsetsymbol{note}{\textdagger}

\begin{icmlauthorlist}
\icmlauthor{Michael W. Dusenberry}{equal,note,google}
\icmlauthor{Ghassen Jerfel}{equal,google,duke}
\icmlauthor{Yeming Wen}{google,toronto}
\icmlauthor{Yi-An Ma}{google,ucsd}
\icmlauthor{Jasper Snoek}{google}
\icmlauthor{Katherine Heller}{google,duke}
\icmlauthor{Balaji Lakshminarayanan}{google}
\icmlauthor{Dustin Tran}{google}
\end{icmlauthorlist}
\icmlaffiliation{google}{Google Brain, Mountain View, USA}
\icmlaffiliation{duke}{Duke University, Durham, USA}
\icmlaffiliation{ucsd}{University of California, San Diego, USA}
\icmlaffiliation{toronto}{University of Toronto, Toronto, CA}

\icmlcorrespondingauthor{Michael W. Dusenberry}{dusenberrymw@google.com}
\icmlcorrespondingauthor{Ghassen Jerfel}{ghassen@google.com}

\icmlkeywords{Bayesian deep learning, machine learning, neural nets}

\vskip 0.3in
]

\printAffiliationsAndNotice{\icmlEqualContribution $^\dagger$Work completed as a Google AI Resident.}

\begin{abstract}
Bayesian neural networks (BNNs) demonstrate promising success in improving the robustness and uncertainty quantification of modern deep learning. However, they generally struggle with underfitting at scale and parameter efficiency. On the other hand, deep ensembles have emerged as alternatives for uncertainty quantification that, while outperforming BNNs on certain problems, also suffer from efficiency issues. It remains unclear how to combine the strengths of these two approaches and remediate their common issues. To tackle this challenge, we propose a rank-1 parameterization of BNNs, where each weight matrix involves only a distribution on a rank-1 subspace. We also revisit the use of mixture approximate posteriors to capture multiple modes, where unlike typical mixtures, this approach admits a significantly smaller memory increase (e.g., only a 0.4\% increase for a ResNet-50 mixture of size 10). We perform a systematic empirical study on the choices of prior, variational posterior, and methods to improve training.
For ResNet-50 on ImageNet, Wide ResNet 28-10 on CIFAR-10/100, and an RNN on MIMIC-III,
rank-1 BNNs achieve state-of-the-art performance across log-likelihood, accuracy, 
and calibration on the test sets and out-of-distribution variants.\footnote{
Code: \url{https://github.com/google/edward2}.
}
\end{abstract}
\glsresetall

\vspace{-4ex}
\section{Introduction}
Bayesian neural networks (BNNs) marginalize over a distribution of neural network models for prediction, allowing for uncertainty quantification and improved robustness in deep learning.
In principle, BNNs can permit graceful failure, signalling when a model does not know what to predict \citep{kendall2017uncertainties,dusenberry2019analyzing},
and can also generalize better to
out-of-distribution examples
\citep{louizos2017multiplicative,malinin2018predictive}.
However, there are two important challenges prohibiting their use in practice. 

First, Bayesian neural networks often underperform on metrics such as accuracy and do not scale as well as simpler baselines \citep{gal2015dropout,lakshminarayanan2017simple,maddox2019simple}. A possible reason is that the best configurations for BNNs remain unknown. What is the best parameterization, weight prior, approximate posterior, or optimization strategy? The flexibility that accompanies these choices makes BNNs broadly applicable, but adds a high degree of complexity.

Second, maintaining a distribution over weights incurs a significant cost both in additional parameters and runtime complexity. Mean-field variational inference~\cite{blundell2015weight}, for example, requires doubling the existing millions or billions of network weights (i.e., mean and variance for each weight). Using an ensemble of size 5, or 5 MCMC samples, requires 5x the number of weights.
In contrast, simply scaling up a deterministic model to match this parameter count
can lead to much better predictive performance on both in- and out-of-distribution data
\citep{recht2019imagenet}.

In this paper, we develop a flexible distribution over neural network weights that achieves state-of-the-art accuracy and uncertainty while being highly parameter-efficient.  We address the first challenge by building on ideas from deep ensembles \citep{lakshminarayanan2017simple}, which work by aggregating predictions from multiple randomly initialized, stochastic gradient descent (SGD)-trained models.
\citet{fort2019deep} identified that deep ensembles' multimodal solutions
provide uncertainty benefits that are distinct and complementary to distributions
centered around a single mode.

We address the second challenge by leveraging recent work that has identified neural network weights as having low effective dimensionality for sufficiently diverse and accurate predictions.
For example,
\citet{li2018measuring} find that the ``intrinsic'' dimensionality of popular architectures can be on the order of hundreds to a few thousand.
\citet{izmailov2019subspace}
perform Bayesian inference on a learned 5-dimensional subspace.
\citet{wen2019batchensemble} apply ensembling on a rank-1 perturbation of each weight matrix
and obtain strong empirical success
without needing to \textit{learn} the subspace.
\citet{swiatkowski2019ktied} apply singular value decomposition post-training
and observe that a rank of 1-3 captures most of the variational posterior's variance.

\textbf{Contributions.}
We propose a rank-1 parameterization of Bayesian neural nets, where each weight matrix involves only a distribution on a rank-1 subspace. This parameterization addresses the above two challenges.
It also allows us to more efficiently leverage heavy-tailed distributions \citep{louizos2017bayesian}, such as Cauchy, without sacrificing predictive performance.
Finally, we revisit the use of mixture approximate posteriors as a simple strategy for aggregating multimodal weight solutions, similar to deep ensembles. Unlike typical ensembles, however, mixtures on the rank-1 subspace involve a significantly reduced dimensionality (for a mixture of size 10 on ResNet-50, it is only 0.4\% more parameters instead of 900\%). Rank-1 BNNs are thus not only parameter-efficient but also scalable, as Bayesian inference is only done over thousands of dimensions.

\Cref{sec:rank-1} performs an empirical study on the choice of prior, variational posterior, and likelihood formulation.
\Cref{sec:rank-1} also presents a theoretical analysis of the expressiveness of rank-1 distributions. 
\Cref{sec:experiments} shows that, 
on ImageNet with ResNet-50, rank-1 BNNs outperform the original network and BatchEnsemble \citep{wen2019batchensemble} on log-likelihood, accuracy, and calibration on both the test set and ImageNet-C.
On CIFAR-10 and 100 with Wide ResNet 28-10, rank-1 BNNs outperform the original model, Monte Carlo dropout, BatchEnsemble, and original BNNs across log-likelihood, accuracy, and calibration on both the test sets and the corrupted versions, CIFAR-10-C and CIFAR-100-C \citep{hendrycks2019benchmarking}.
Finally, on the MIMIC-III electronic health record (EHR) dataset \citep{johnson2016} with LSTMs, rank-1 BNNs outperform deterministic and stochastic baselines from \citet{dusenberry2019analyzing}.

\vspace{-1.5ex}
\section{Background}
\label{sec:background}

\subsection{Variational inference for Bayesian neural networks}
Bayesian neural networks posit a prior distribution over weights $p(\mathbf{W})$ of a network architecture. 
Given a dataset $(\mathbf{X}, \mathbf{y})$ of $N$ input-output pairs, we perform approximate Bayesian inference using variational inference: we select a family of variational distributions $q(\mathbf{W})$ with free parameters and then minimize the Kullback-Leibler (KL) divergence from $q(\mathbf{W})$ to the true posterior $p(\mathbf{W}\mid \mathbf{X},\mathbf{y})$ \citep{jordan1999introduction}. Taking a minibatch of size $B$, this is equivalent to minimizing the loss function,
\begin{equation*}
-\frac{N}{B} \sum_{b=1}^B \mathbb{E}_{q(\mathbf{W})} [ \log p(y_b\mid \mathbf{x}_b, \mathbf{W}) ] + \operatorname{KL}( q(\mathbf{W}) \| p(\mathbf{W}) ),
\end{equation*}
with respect to the parameters of $q(\mathbf{W})$. This loss function is an upper bound on the negative log-marginal likelihood $-\log p(\mathbf{y}\mid \mathbf{X})$ and can be interpreted as the model's approximate description length \citep{hinton1993keeping}.

In practice, Bayesian neural nets often underfit, mired by complexities in both the choice of prior and approximate posterior, and in stabilizing the training dynamics involved by the loss function (e.g., posterior collapse \citep{bowman2015generating}) and the additional variance from sampling weights to estimate the expected log-likelihood.
In addition, note even the simplest solution of a fully-factorized normal approximation incurs a 2x cost in the typical number of parameters.

\vspace{-1ex}
\subsection{Ensemble \& BatchEnsemble}

Deep ensembles \citep{lakshminarayanan2017simple}
are a simple and effective method for ensembling, where 
one trains multiple copies of a network
and then makes predictions by aggregating 
the individual models to form a mixture distribution.
However, this comes at the cost of training and predicting with multiple copies of network parameters.

BatchEnsemble \citep{wen2019batchensemble} is a parameter-efficient extension that ensembles over a low-rank subspace.
Let the ensemble size be $K$ and, for each layer, denote the original weight matrix $\mathbf{W}\in \mathbb{R}^{m\times d}$, which will be shared across ensemble members.
Each ensemble member $k$ owns a tuple of trainable vectors $\mathbf{r}_k$ and $\mathbf{s}_k$ of size $m$ and $d$ respectively. BatchEnsemble defines $K$ ensemble weights: each is
\begin{equation*}
\mathbf{W}'_k = \mathbf{W} \circ \mathbf{F}_k, \text{ where  } \mathbf{F}_k = \mathbf{r}_k \mathbf{s}_k^{\top} \in \mathbb{R}^{m\times d},
\end{equation*}
and $\circ$ denotes element-wise product.
BatchEnsemble's forward pass can be rewritten, where for a given layer,
\begin{align}
\begin{split}
\mathbf{y} &= \phi \left( {\mathbf{W}_k'} \mathbf{x} \right) 
= \phi \left( \left(\mathbf{W} \circ r_k s_k^\top \right) \mathbf{x} \right) 
\\
&= \phi \left(\left(\mathbf{W}(\mathbf{x} \circ s_k) \right) \circ r_k \right),
\end{split}
\label{eq:vectorization}
\end{align}
where $\phi$ is the activation function, and $\mathbf{x}\in\mathbb{R}^d, \mathbf{y}\in\mathbb{R}^m$ is a single example. In other words, the rank-1 vectors $\mathbf{r}_k$ and $\mathbf{s}_k$ correspond to elementwise multiplication of input neurons and pre-activations. This admits efficient vectorization as we can replace the vectors $\mathbf{x}$, $\mathbf{r}_k$, and $\mathbf{s}_k$ with matrices where each row of $\mathbf{X} \in \mathbb{R}^{B \times d}$ is a batch element and each row of $\mathbf{R} \in \mathbb{R}^{B \times m}$ and $\mathbf{S} \in \mathbb{R}^{B \times d}$ is a choice of ensemble member:
$\phi\left(\left((\mathbf{X} \circ \mathbf{S})\mathbf{W}^\top \right) \circ \mathbf{R} \right)$.
This vectorization extends to other linear operators such as convolution and recurrence.

\vspace{-1.5ex}
\section{Rank-1 Bayesian Neural Nets}
\label{sec:rank-1}
\vspace{-0.5ex}
Building on \Cref{eq:vectorization}, we introduce a rank-1 parameterization of Bayesian neural nets. We then empirically study choices such as the  prior and variational posterior.
\vspace{-1ex}
\subsection{Rank-1 Weight Distributions}

Consider a Bayesian neural net with rank-1 factors: parameterize every $m\times d$ weight matrix $\mathbf{W}' = \mathbf{W} \circ \mathbf{r}\mathbf{s}^\rT$, where the factors $\mathbf{r}$ and $\mathbf{s}$ are $m$ and $d$-vectors respectively.
We place priors on $\mathbf{W}'$ by placing priors on $\mathbf{r}$, $\mathbf{s}$, and $\mathbf{W}$. Upon observing data, we compute for $\mathbf{r}$ and $\mathbf{s}$ (the \textit{rank-1 weight distributions}), while treating $\mathbf{W}$ as deterministic.

\textbf{Variational Inference.}
For training, we apply variational EM where we perform approximate posterior inference over $\mathbf{r}$ and $\mathbf{s}$, and point-estimate the weights $\mathbf{W}$ with maximum likelihood. The loss function is
\begin{align}
\mathcal{L} &=
-\frac{N}{B} \sum_{b=1}^B \mathbb{E}_{q(\mathbf{r})q(\mathbf{s})} [ \log p(y_b\mid \mathbf{x}_b, \mathbf{W}, \mathbf{r}, \mathbf{s}) ]
\label{eq:loss}\\
&\hspace{1em}+ \operatorname{KL}( q(\mathbf{r}) \| p(\mathbf{r}) ) + \operatorname{KL}( q(\mathbf{s}) \| p(\mathbf{s}) ) - \log p(\mathbf{W}),
\nonumber 
\end{align}
where the parameters are $\mathbf{W}$ and the variational parameters of $q(\mathbf{r})$ and $q(\mathbf{s})$. In all experiments, we set the prior $p(\mathbf{W})$ to a zero-mean normal with fixed standard deviation, which is equivalent to an L2 penalty for deterministic models.

Using rank-1 distributions enables significant variance reduction: weight sampling only comes from the rank-1 variational distributions rather than over the full weight matrices (tens of thousands compared to millions). In addition, 
\Cref{eq:vectorization} holds, enabling sampling of new $\mathbf{r}$ and $\mathbf{s}$ vectors for each example and for arbitrary distributions $q(\mathbf{r})$ and $q(\mathbf{s})$.

\textbf{Multiplicative or Additive Perturbation?}
A natural question is whether to use a multiplicative or additive update. For location-scale family distributions, multiplication and addition only differ in the location parameter and are invariant under a scale reparameterization.
For example: let $r_i\sim \operatorname{Normal}(\mu, \sigma^2)$ and for simplicity, ignore $\mathbf{s}$; then
\begin{equation*}
w_{ij}r_i = w_{ij}(\mu_i + \sigma_i\epsilon_i) = w_{ij}\mu_i + r_i',
\end{equation*}
where $r_i'\sim\operatorname{Normal}(0, \sigma_i'^2)$ and $\sigma_i'=w_{ij}\sigma_i$. Therefore additive perturbations only differ in an additive location parameter ($+x\circ s\circ r$). An additive location is often redundant as, when vectorized under \Cref{eq:vectorization}, it's subsumed by any biases and skip connections.

\vspace{-1ex}
\subsection{Rank-1 Priors Are Hierarchical Priors}
\label{sub:hierarchical}

\begin{figure}[t]
\vspace{-1ex}
\centering
\includegraphics[width=0.8\columnwidth]{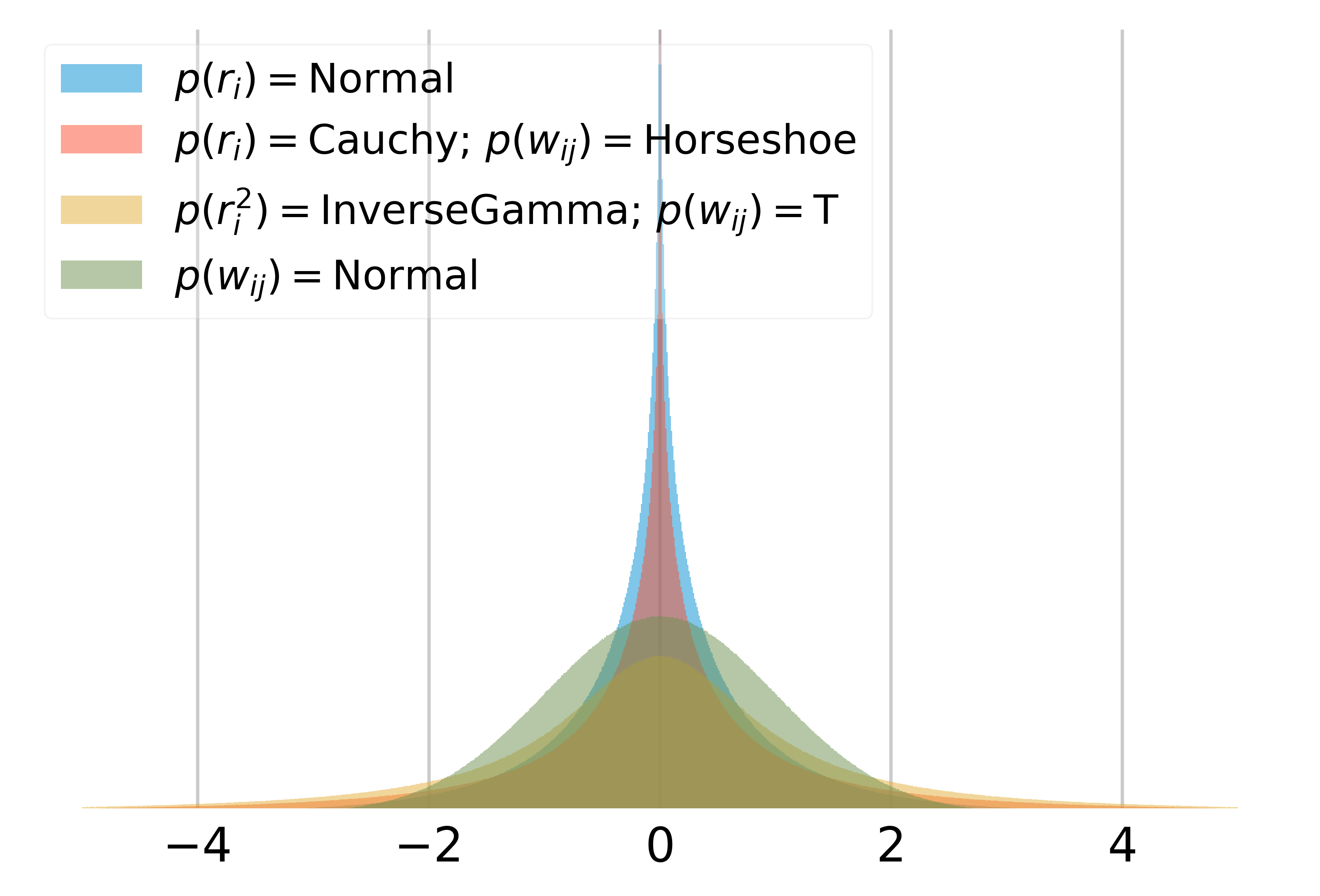}
\caption{Induced weight priors. The distribution of a weight element is ${w}'_{ij} = w_{ij} r_i s_j$, where $w_{ij}\sim\mathcal{N}(0, \cdot)$, $s_j$ is fixed at $1$, and $r_i$ is varied. \blue{Normal} and \red{Cauchy} priors on $r_i$ both encourage sparse weight posteriors: Cauchy has less mass around 0 and heavier tails. \orange{Inverse-Gamma} $r_i^2$ induces a Student-T weight prior unlike a \green{normal weight prior}.
}
\vspace{-1ex}
\label{fig:weights}
\end{figure}

Priors over the rank-1 factors can be viewed as hierarchical priors on the weights in a noncentered parameterization, that is, where the distributions on the weights and scale factors are independent. This removes posterior correlations between the weights which can be otherwise difficult to approximate \citep{ingraham2017variational,louizos2017bayesian}.
We examine choices for priors based on this connection.

\begin{figure}[t]
\centering
\includegraphics[width=\columnwidth]{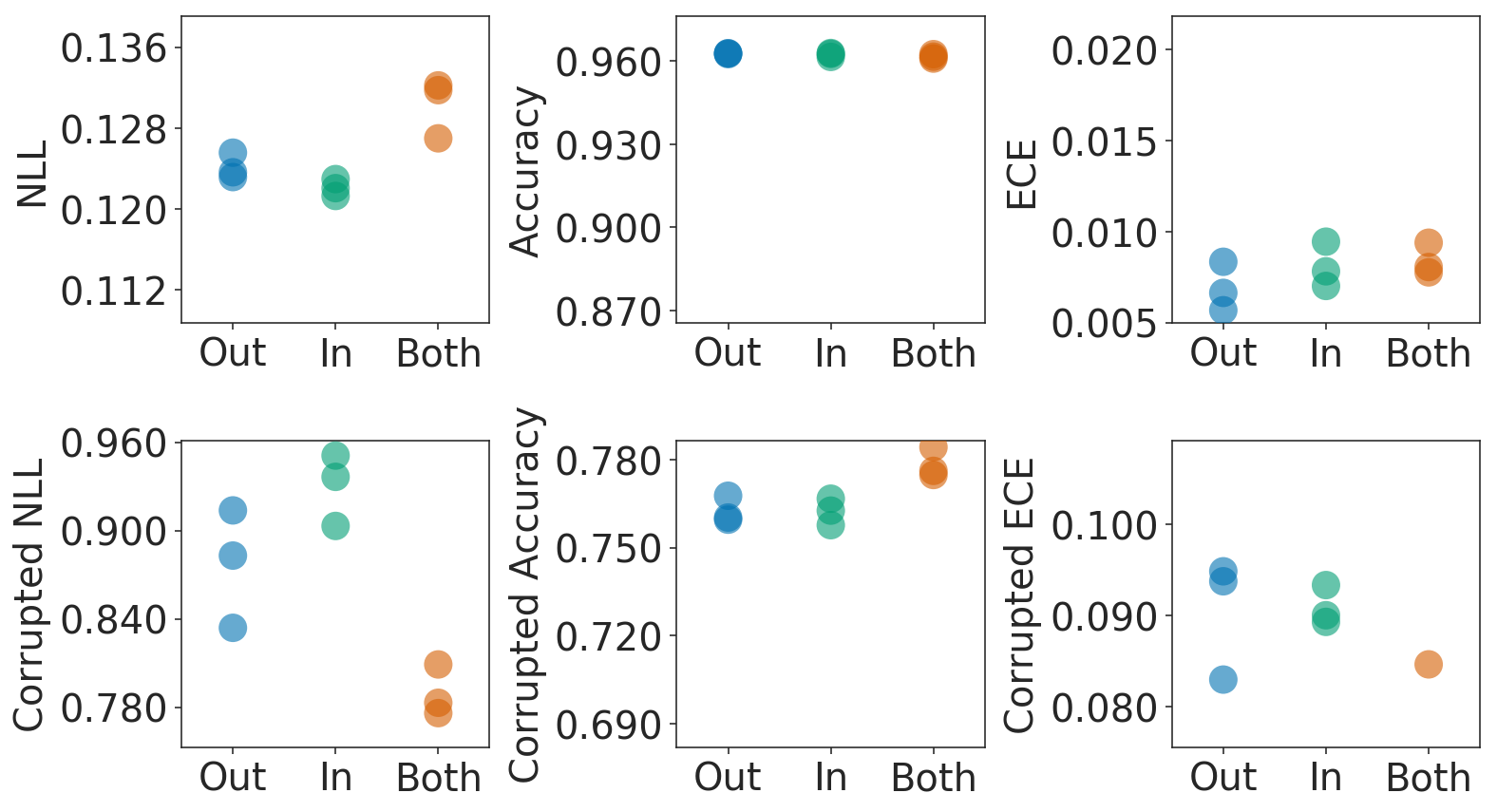}
\caption{Placing distributions over $\mathbf{r}$ (\blue{output}), $\mathbf{s}$ (\green{input}), and \orange{both}, evaluated over three runs on the CIFAR-10 test set and CIFAR-10-C.
The best setup differs on the test set, while priors over both vectors generalize better on corruptions.}
\label{fig:hierarchical}
\vspace{-2ex}
\end{figure}

\textbf{Hierarchy across both input and output neurons}.
Typical hierarchical priors for BNNs are Gaussian-scale mixtures, which take the form
\begin{equation*}
p(\mathbf{W}') = \int \mathcal{N}(\mathbf{W}'\mid 0, \mathbf{r}^2\sigma^2) p(\mathbf{r})p(\sigma^2)~\text{d}\mathbf{r}\,\text{d}\sigma^2,
\end{equation*}
where $\mathbf{r}$ is a vector shared across rows or columns and $\sigma$ is a global scale across all elements. Settings of $\mathbf{r}$ and $\sigma$ lead to well-known distributions (\Cref{fig:weights}): Inverse-Gamma variance induces a Student-t distribution on $\mathbf{W'}$; half-Cauchy scale induces a horseshoe distribution \citep{carvalho2009handling}.
For rank-1 priors, the induced weight distribution is
\begin{equation}
p(\mathbf{W}') = \iint \mathcal{N}(\mathbf{W}'\mid 0, (\mathbf{r}\mathbf{s}^\rT\sigma)^2) p(\mathbf{r}) p(\mathbf{s})~\text{d}\mathbf{r}\,\text{d}\mathbf{s},
\label{eq:rank1-weight-prior}
\end{equation}
where $\mathbf{r}$ is a vector shared across columns; $\mathbf{s}$ is a vector shared across rows; and $\sigma$ is a scalar hyperparameter.

To better understand the importance of hierarchy, \Cref{fig:hierarchical} examines three settings under the best model on CIFAR-10 (\Cref{subsec:experiments-cifar10}): priors (paired with non-degenerate posteriors) on (1) only the vector $\mathbf{s}$ that is applied to the layer's inputs, (2) only the vector $\mathbf{r}$ that is applied to the outputs, and (3) the default of both $\mathbf{s}$ and $\mathbf{r}$.
The presence of a prior corresponds to a mixture of Gaussians with tuned, shared mean and standard deviation,
and the corresponding approximate posterior is a mixture of Gaussians with learnable parameters; the absence of a prior indicates point-wise estimation. L2 regularization on the point-estimated $\mathbf{W}$ is also tuned. 

Looking at test performance, we find that the settings perform comparably on accuracy and differ slightly on test NLL and ECE. More interestingly, when we look at the corruptions task, the hierarchy of priors across \textit{both} vectors outperforms the others on all three metrics, suggesting improved generalization. We hypothesize that the ability to modulate the uncertainty of both the inputs and outputs of each layer assists in handling distribution shift.

\textbf{Cauchy priors: Heavy-tailed real-valued priors.}
Weakly informative priors such as the Cauchy are often preferred for robustness as they concentrate less probability at the mean thanks to heavier tails \citep{gelman2006prior}.
The heavy tails encourage the activation distributions to be farther apart at training time, reducing the mismatch when passed out-of-distribution inputs.
However, the exploration of heavy-tailed priors has been mostly limited to half-Cauchy \citep{carvalho2010horseshoe} and log-uniform priors \citep{kingma2015variational} on the scale parameters,
and there has been a lack of empirical success beyond compression tasks. These priors are often justified by the assumption of a positive support for scale distributions.
However, in a non-centered parametrization, such restriction on the support is unnecessary and we find that real-valued scale priors typically outperform positive-valued ones (\Cref{sub:real}).
Motivated by this, we explore in \Cref{sec:experiments} the improved generalization and uncertainty calibration provided by Cauchy rank-1 priors.

\vspace{-1ex}
\subsection{Choice of Variational Posterior}

\textbf{Role of Mixture Distributions.}
Rank-1 BNNs admit few stochastic dimensions, making mixture distributions over weights more feasible to scale. For example, a mixture approximate posterior with $K=10$ components for ResNet-50 results in an 0.4\% increase in parameters, compared with a 900\% increase in deep ensembles. A natural question is: to what extent can we scale $K$ before there are diminishing returns? \Cref{fig:num_models} examines the best-performing rank-1 model under our CIFAR-10 setup, varying the mixture size $K\in\{1,2, 4, 8, 16\}$. For each, we tune over the total number of training epochs, and measure NLL, accuracy, and ECE on both the test set and CIFAR-10-C corruptions dataset. As the number of mixture components increases from 1 to 8, the performance across all metrics increases. At $K=16$, however, there is a decline in performance. Based on our findings, all experiments in \Cref{sec:experiments} use $K=4$.

For mixture size $K=16$, we suspect the performance is a
result of the training method and hardware memory constraints.
Namely, we start with a batch of $B$ examples and duplicate it $K$ times so that each mixture component applies a forward pass for each example; the total batch size supplied to the model is $B\cdot K$.
We keep this total batch size constant as we increase $K$ in order to maintain constant memory. This implies that as the number of mixture components increases, the batch size $B$ of new data points decreases. We suspect alternative implementations such as sampling mixture components may enable further scaling.

\textbf{Role of Non-Degenerate Components.}
To understand the role of non-degenerate distributions (i.e., distributions that do not have all probability mass at a single point), note that BatchEnsemble can be interpreted as using a mixture of Dirac delta components.
\Cref{sec:experiments} compares to BatchEnsemble in depth, providing broad evidence that mixtures consistently improve results (particularly accuracy), and using non-degenerate components further lowers probabilistic metrics (NLL and ECE) as well as improves generalization to out-of-distribution examples.

\begin{figure}[!t]
\centering
\includegraphics[width=\columnwidth]{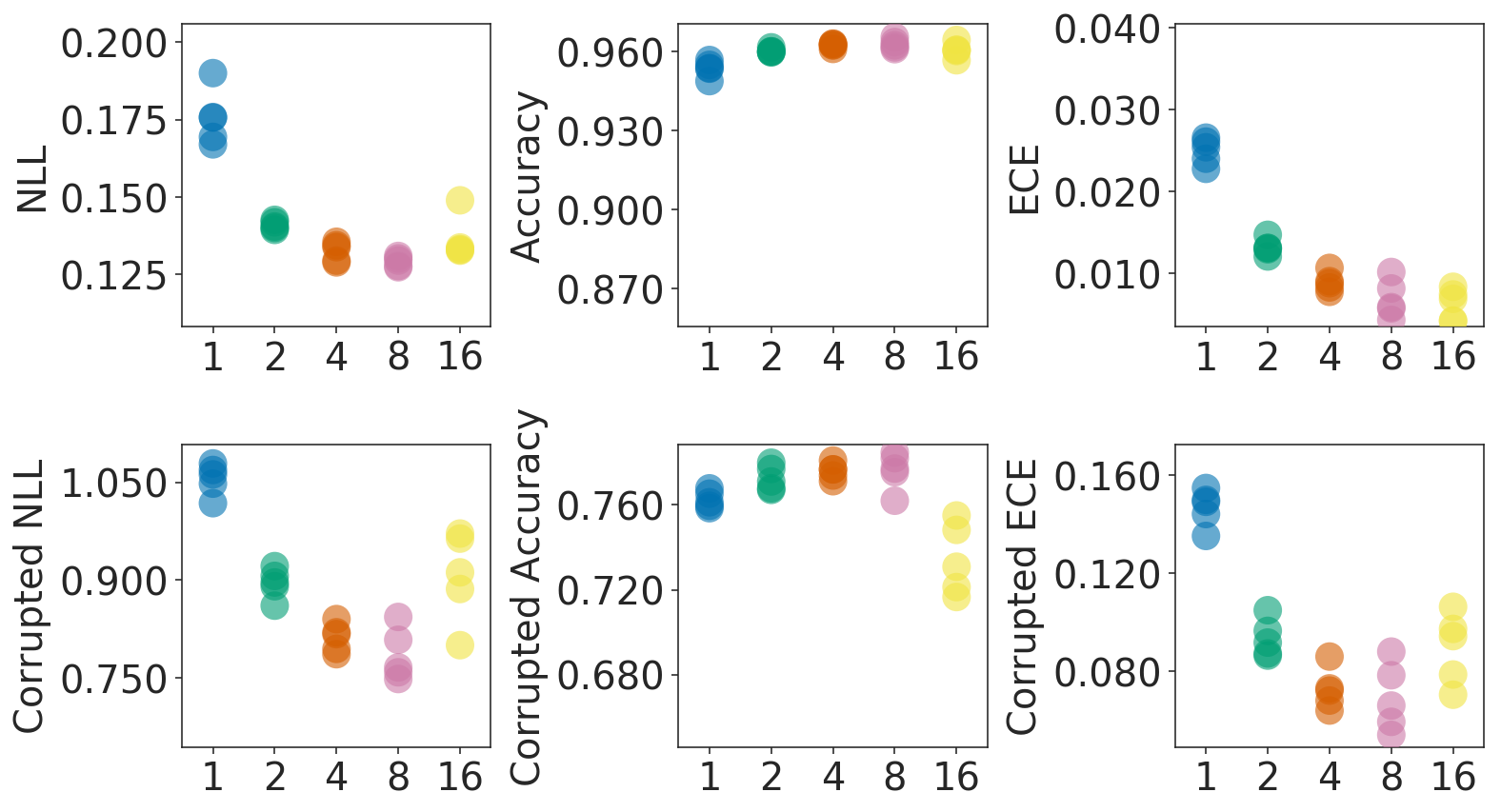}
\vspace{-2ex}
\caption{Varying the number of mixture components in the rank-1 mixture of Gaussians posteriors, evaluated over five runs on the CIFAR-10 test set and CIFAR-10-C corrupted dataset. Increasing the number of components yields improved performance up to a limit.
}
\label{fig:num_models}
\vspace{-4ex}
\end{figure}

\vspace{-1ex}
\subsection{Log-likelihood: Mixture or Average?}
\label{sub:likelihood}

When using mixture distributions as the approximate posterior, the expected log-likelihood in \Cref{eq:loss} involves an average over all mixture components.
By Jensen's inequality, one can get a tighter bound on the log-marginal likelihood by using the log-mixture density,
\begin{align*}
\log \frac{1}{K}\sum_{k=1}^K p(y_n\mid \mathbf{x}_n, \theta_k) 
&\geq
\frac{1}{K} \sum_{k=1}^K  \log p(y_n\mid \mathbf{x}_n, \theta_k),
\end{align*}
where $\theta_k$ are per-component parameters.
The log-mixture likelihood is typically preferred over the average as it is guaranteed to provide at least as good a bound on the log-marginal. Appendix \ref{sec:nll} contains a further derivation of the various choices of log-likelihood losses for such models.

However, deep ensembles when interpreted as a mixture distribution correspond to using the average as the loss function: for the gradient of parameters $\theta_{k'}$ in mixture component $k'$,
\begin{align*}
\nabla_{\theta_{k'}}\log \frac{1}{K}\sum_{k=1}^K p(y\mid \mathbf{x}, \theta_k) 
&= \frac{\nabla p(y\mid \mathbf{x}, \theta_{k'})}{K^{-1}\sum_{k=1}^K p(y\mid \mathbf{x}, \theta_k)}
\\
\nabla_{\theta_{k'}} \frac{1}{K} \sum_{k=1}^K  \log p(y\mid \mathbf{x}, \theta_k)
&=
\frac{1}{K} \frac{1}{\nabla p(y\mid \mathbf{x}, \theta_{k'})}.
\end{align*}
Therefore, while the log-mixture likelihood is an upper bound, it incurs a communication cost where each mixture component's gradients are a function of how well the other mixture components fit the data.
This communication cost prohibits the use of log-mixture likelihood as a loss function for deep ensembles,
where randomly initialized ensemble members are trained independently.

We wonder whether deep ensembles' lack of communication across mixture components and relying purely on random seeds for diverse solutions is in fact better.
With rank-1 priors, we can do either with no extra cost: \Cref{fig:mixture} compares the two using the best rank-1 BNN hyperparameters on CIFAR-10.
Note that we always use the log-mixture likelihood for evaluation.
While the training metrics in \Cref{fig:mixture} are comparable, the log-mixture likelihood generalizes worse than the average log-likelihood,
and the individual mixture components also generalize worse.
It seems that, at least for misspecified models such as overparametrized neural networks, training a looser bound on the log-likelihood
leads to improved predictive performance. We conjecture that this might simply be a case of ease of optimization allowing the model to explore more distinct modes throughout the training procedure.

\begin{figure}[!t]
\vspace{-1ex}
\centering
\begin{subfigure}[t]{0.32\columnwidth}
\centering
\includegraphics[width=\columnwidth]{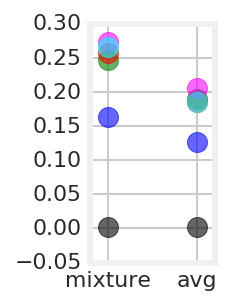}
\caption{NLL}
\end{subfigure}\hfill
\begin{subfigure}[t]{0.29\columnwidth}
\centering
\includegraphics[width=\columnwidth]{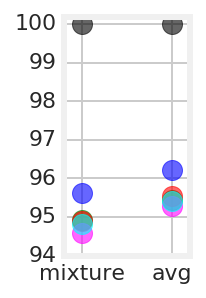}
\caption{Accuracy}
\end{subfigure}\hfill
\begin{subfigure}[t]{0.32\columnwidth}
\centering
\includegraphics[width=\columnwidth]{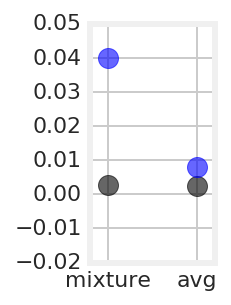}
\caption{ECE}
\end{subfigure}
\vspace{-1ex}
\caption{Training with a log-mixture likelihood vs an average per-component log-likelihood.
\blue{Blue} is averaged (test) performance; \red{co}\green{lo}\sky{r}\textcolor{magenta}{s} are individual components; \textbf{black} is averaged (train) performance. Training metrics are identical but the average consistently outperforms on the test set.
}
\label{fig:mixture}
\vspace{-3.5ex}
\end{figure}

\vspace{-1ex}
\subsection{Ensemble Diversity}
The diversity of predictions returned by different members of an ensemble is an important indicator of the quality of uncertainty quantification \citep{fort2019deep} and of the robustness of the ensemble \citep{pang2019improving}. Following \citet{fort2019deep}, \Cref{fig:diversity} examines the disagreement of rank-1 BNNs and BatchEnsemble members against accuracy and log-likelihood, on test data.

We quantify diversity by the fraction of points where discrete predictions differ between two members, averaged over all pairs.
This disagreement measure is normalized 
by $(1 - \operatorname{acc})$ to account for the fact that the lower the accuracy of a member, the more random its predictions can be. Unsurprisingly, \Cref{fig:diversity} demonstrates a negative correlation between accuracy and diversity for both methods. For the same or higher predictive performance, rank-1 BNNs achieve a higher degree of ensemble diversity than BatchEnsemble on both CIFAR-10 and CIFAR-100.

This can be attributed to the non-degenerate posterior distribution around each mode of the mixture, which can better handle modes that are closest together. In fact, a deterministic mixture model could place multiple modes within a single valley in the loss landscape parametrized by weights. Accordingly, the ensemble members are likely to collapse on near-identical modes in the function space. On the other hand, a mixture model that can capture the uncertainty around each mode might be able to detect a single `wide' mode, as characterized by large variance around the mean.
Overall, the improved diversity result confirms our intuition about the necessity of combining local (near-mode) uncertainty with a multimodal representation in order to improve the predictive performance of mode averaging.

\begin{figure}[!t]
\vspace{-1ex}
\includegraphics[width=\columnwidth]{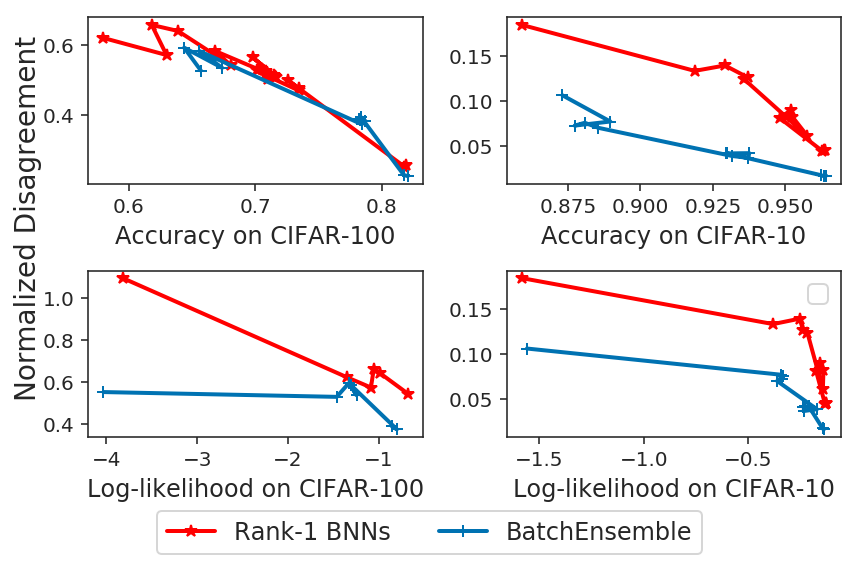}
\vspace{-4ex}
\caption{Disagreement versus accuracy and log-likelihood over consecutive model checkpoints, at the end of training, for rank-1 BNNs and BatchEnsemble on CIFAR-10/100. Rank-1 BNNs demonstrate a higher diversity while achieving better predictive performance than BatchEnsemble.}
\vspace{-2.5ex}
\label{fig:diversity}
\end{figure}

\vspace{-1ex}
\subsection{Expressiveness of Rank-1 Distribution}
A natural question is how expressive a rank-1 distribution is. \Cref{theorem:thm_informal} below demonstrates that the rank-1 perturbation encodes a wide range of perturbations in the original weight matrix $\mathbf{W}$.
We prove that, for a fully connected neural network, the rank-1 parameterization has the same local variance structure in the score function as a full-rank's. 
\begin{theorem}[Informal]
\label{theorem:thm_informal}
In a fully connected neural network of any width and depth, let $\mathbf{W}_*$ denote a local minimum associated with a score function over a dataset.
Assume that the full-rank perturbation on the weight matrix in layer $h$ has the multiplicative covariance structure that 
\begin{align*}
&\Ep{\mathbf{W}^{(h)}}{\lrp{\mathbf{W}^{(h)}-\mathbf{W}^{(h)}_*}_{i,j} \lrp{\mathbf{W}^{(h)}-\mathbf{W}^{(h)}_*}_{k,l}}  \\
&= {\mathbf{W}^{(h)}_*}_{i,j} \mathbf{\Sigma}_{j,k} {\mathbf{W}^{(h)}_*}_{k,l},
\end{align*}
for some symmetric positive semi-definite matrix $\mathbf{\Sigma}$.
Let $\mathbf{s}_*^{(h)}$ denote a column vector of ones.
Then if the rank-1 perturbation has covariance 
\begin{align*}
\Ep{\mathbf{s}^{(h)}}{ \lrw{\lrp{\mathbf{s}^{(h)}-\mathbf{s}^{(h)}_*} \lrp{\mathbf{s}^{(h)}-\mathbf{s}^{(h)}_*}}^\rT }
= \mathbf{\Sigma},     
\end{align*}
the score function has the same variance around the local minimum.
\vspace{-1ex}
\end{theorem}
\Cref{theorem:thm_informal} demonstrates a correspondence between the covariance structure in the perturbation of $\mathbf{W}$ and that of $\mathbf{s}$.
Since $\mathbf{\Sigma}$ can be any symmetric positive semi-definite matrix, our rank-1 parameterization can efficiently encode a wide range of fluctuations in $\mathbf{W}$.
In particular, it is especially suited for multiplicative noise as advertised.
If the covariance of $\lrp{\mathbf{W}-\mathbf{W}_*}$ is proportional to $\mathbf{W}_* \otimes \mathbf{W}_*^\rT$ itself,
then we can simply take the covariance of $\lrp{\mathbf{s}-\mathbf{s}_*}$ to be identity.
See \Cref{appnd:thm} for a formal version of \Cref{theorem:thm_informal}.

\vspace{-1.5ex}
\section{Experiments}
\label{sec:experiments}
In this section, we show results on image classification and electronic health record classification tasks: ImageNet, CIFAR-10, CIFAR-100, their corrupted variants \citep{hendrycks2019benchmarking}, and binary mortality prediction with the MIMIC-III EHR dataset \citep{johnson2016}.
For ImageNet, we use a ResNet-50 baseline as it's the most commonly benchmarked model \citep{he2016deep}.
For CIFAR, we use a Wide ResNet 28-10 baseline as it's a simple architecture that achieves 95\%+ test accuracy on CIFAR-10
with little data augmentation
\citep{zagoruyko2016wide}.
For MIMIC-III, we use recurrent neural networks (RNNs) based on the setup in \citet{dusenberry2019analyzing}.

\textbf{Baselines.} For the image classification tasks, we reproduce and compare to baselines with equal parameter count: deterministic (original network);
Monte Carlo dropout \citep{gal2015dropout}; and BatchEnsemble \citep{wen2019batchensemble}. Although 2x the parameter count of other methods, we also tune a vanilla BNN baseline for CIFAR that uses Gaussian priors and approximate posteriors over the full set of weights with Flipout \citep{wen2018flipout} for estimating expectations. We additionally include reproduced results for two deep ensemble \citep{lakshminarayanan2017simple} setups: one with an equal parameter count for the entire ensemble, and one with $K$ times more parameters for an ensemble of $K$ members.

For the EHR task, we reproduce and compare to the LSTM-based RNN baselines from \citet{dusenberry2019analyzing}: deterministic; Bayesian Embeddings (distributions over the embeddings); and Fully Bayesian (distributions over all parameters). We additionally compare against BatchEnsemble, and include reproduced results for deep ensembles.

\vspace{-1ex}
\subsection{ImageNet and ImageNet-C}

ImageNet-C \citep{hendrycks2019benchmarking} applies a set of 15 common visual corruptions to ImageNet \citep{Deng2009ImageNetAL} with varying intensity values (1-5). It was designed to benchmark the robustness to image corruptions. \Cref{table:imagenet} presents results for negative log-likelihood (NLL), accuracy, and expected calibration error (ECE) on the standard ImageNet test set, as well as on ImageNet-C.
We also include mean corruption error (mCE)
\citep{hendrycks2019benchmarking}. \Cref{fig:imagenetc} examines out-of-distribution performance in more detail by plotting the mean result across corruption types for each corruption intensity.

BatchEnsemble improves accuracy (but not NLL or ECE) over the deterministic baseline. Rank-1 BNNs, which involve non-degenerate mixture distributions 
over BatchEnsemble, further improve results across all metrics.

Rank-1 BNN's results are comparable in terms of test NLL and accuracy to previous works which scaled up BNNs to ResNet-50. \citet{zhang2019cyclical} use 9 MCMC samples and
report 77.1\% accuracy and 0.888 NLL; and \citet{heek2019bayesian} use 30 MCMC samples and report 77.5\% accuracy and 0.883 NLL.
Rank-1 BNNs have a similar parameter count to deterministic ResNet-50, instead of incurring a 9-30x memory cost, and use a single MC sample from each mixture component by default.\footnote{
\citet{heek2019bayesian} also report results using a single sample: 74.2\% accuracy, 1.08 NLL. Rank-1 BNNs outperform.}
Rank-1 BNNs also do not use techniques such as tempering, which trades off uncertainty
in favor of predictive performance. We predict rank-1 BNNs may outperform these methods if measured by ECE or out-of-distribution performance.

\subsection{CIFAR-10 and CIFAR-10-C}
\label{subsec:experiments-cifar10}

\Cref{table:cifar10} demonstrates results with respect to NLL, accuracy, and ECE on the CIFAR-10 test set, and the same three metrics on CIFAR-10-C. \Cref{fig:cifar100c} examines out-of-distribution performance as the skew intensity (severity of corruption) increases.
\Cref{sec:cifar10-c} contains a clearer comparison.

On CIFAR-10, both Gaussian and Cauchy rank-1 BNNs outperform similarly-sized baselines in terms of NLL, accuracy, and ECE. 
The improvement on NLL and ECE is more significant than that on accuracy, which highlights the improved uncertainty measurement.
An even more significant improvement is observed on CIFAR-10-C: the NLL improvement from BatchEnsemble is 1.02 to 0.74; accuracy increases by 3.7\%; and calibration decreases by 0.05. This, in addition to \Cref{fig:cifar10c-results} in the Appendix, is clear evidence of improved generalization and uncertainty calibration for rank-1 BNNs, even under distribution shift.

The vanilla BNN baseline underfits compared to the deterministic baseline, despite an extensive search over hyperparameters.
We suspect this is a result of the difficulty of optimization given weight variance
and overregularization due to priors over all weights. 
Rank-1 BNNs do not face these issues and consistently outperform vanilla BNNs.

In comparison to deep ensembles \citep{lakshminarayanan2017simple}, rank-1 BNNs outperform the similarly-sized ensembles on accuracy, while only underperforming deep ensembles that have 4 times the number of parameters. 
Rank-1 BNNs still perform better on in-distribution ECE, as well as on accuracy and NLL under distribution shift.

Rank-1 BNN's results are similar to SWAG \citep{maddox2019simple} and Subspace Inference \citep{izmailov2019subspace} despite those having a significantly stronger deterministic baseline and 5-25x parameters:
SWAG gets 96.4\% accuracy, 0.112 NLL, 0.009 ECE; 
Subspace Inference gets 96.3\% accuracy, 0.108 NLL, and does not report ECE; their deterministic baseline gets 96.4\% accuracy, 0.129 NLL, 0.017 ECE (vs. our 96.0\%, 0.159, 0.023). They don't report out-of-distribution performance. 
Rank-1 outperforms on accuracy and underperforms on NLL.

\begin{figure}[!t]
\vspace{-1.5ex}
\centering
\includegraphics[width=\columnwidth]{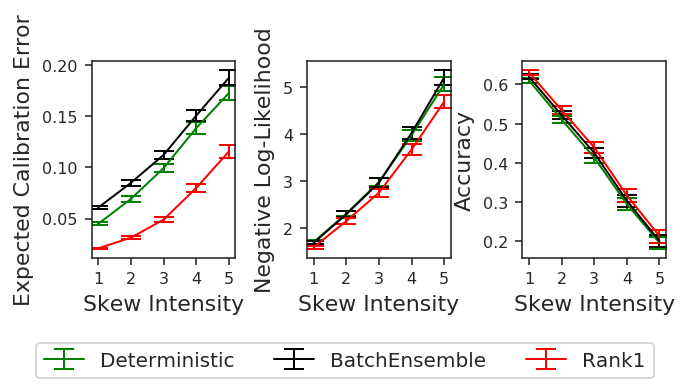}
\vspace{-4ex}
\caption{Out-of-distribution performance using ImageNet-C with ResNet-50. We plot NLL, accuracy, and ECE for varying corruption intensities; each result is the mean performance over 10 runs and over 15 corruption types. The error bars represent the standard deviation across corruption types. \Cref{fig:imagenet-c-results} elaborates on these results in the Appendix. Rank-1 BNNs (\red{red}) perform best across all metrics.}
\label{fig:imagenetc}
\vspace{-3ex}
\end{figure}

\begin{figure}[t]
\vspace{-1.5ex}
\centering
\includegraphics[width=\columnwidth]{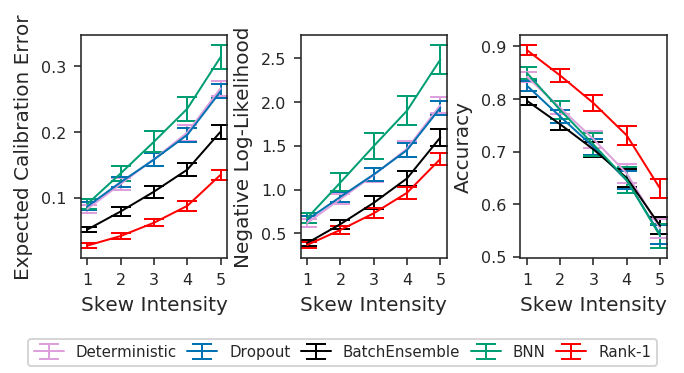}
\includegraphics[width=\columnwidth]{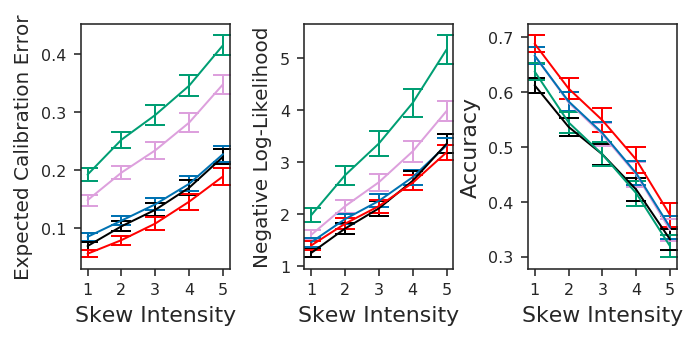}
\vspace{-4ex}
\caption{Out-of-distribution performance using CIFAR-10-C \textbf{(top)} and CIFAR-100-C \textbf{(bottom)} with WRN-28-10. We plot NLL, accuracy, and ECE for varying corruption intensities; each result is the mean performance over 10 runs and 15 corruption types. The error bars represent a fraction 
of the standard deviation across corruption types.
Rank-1 BNNs (\red{red}) perform best across all metrics.
}
\label{fig:cifar100c}
\vspace{-4ex}
\end{figure}

\begin{figure*}[!tb]
\vspace{-1.5ex}
\begin{minipage}{1.0\textwidth}
\centering
\begin{tabular}{clcccccc}
\toprule
\multicolumn{2}{c}{{\small Method}} & 
{\small NLL($\downarrow$)} & 
{\small Accuracy($\uparrow$)} & 
{\small ECE($\downarrow$)} & 
{\small cNLL / cA / cECE}& 
{\small mCE($\downarrow$)}
& {\small \# Parameters} \\
\midrule
\multicolumn{2}{c}{Deterministic} & 0.943 & 76.1 & 0.0392 & 3.20 / 40.5 / 0.105
& 75.34
& 25.6M \\
\multicolumn{2}{c}{BatchEnsemble} & 0.951 & 76.5 & 0.0532 & 3.23 / 41.4 / 0.120
& 74.14
& 25.8M \\
\midrule
\multirow{2}{*}{\textbf{Rank-1 BNN}}
& Gaussian 
& \textbf{0.886} & \textbf{77.3} & \textbf{0.0166} & \textbf{2.95
/ 42.9 / 0.054}
& \textbf{72.12}
& 26.0M \\
& $\text{Cauchy}^{\text{(4 samples)}}$
& 0.897 & 77.2 & 0.0192 & 2.98
/ 42.5 / 0.059
& 72.66
& 26.0M \\
\midrule
Deep Ensembles & ResNet-50 & \textbf{0.877} & \textbf{77.5} & 0.0305  & 2.98 / 42.1  / 0.050  & 73.25 & 146.7M \\
\midrule
MCMC BNN$^1$ & 9 MC samples & 0.888 & 77.1 & - & - & - & 230.4 \\
MCMC BNN$^2$ & 30 MC samples & 0.883 & \textbf{77.5} & - & - & - &  768M \\
\bottomrule
\end{tabular}
\vspace{-1.5ex}
\captionof{table}{
Results for ResNet-50 on ImageNet: negative log-likelihood (lower is better), accuracy (higher is better), and expected calibration error (lower is better). cNLL, cA, and cECE are NLL, accuracy, and ECE averaged over ImageNet-C's corruption types and intensities. mCE is mean corruption error. Results are averaged over 10 seeds, and over 1 weight sample (per mixture component, per seed) for the Gaussian rank-1 BNN, and 4 samples for Cauchy.
We include results for $^1$\citet{zhang2019cyclical} and $^2$\citet{heek2019bayesian}.
Rank-1 BNNs
consistently outperform baselines across all metrics.
}
\label{table:imagenet}
\end{minipage}

\vspace{0.5ex}
\begin{minipage}{1.0\textwidth}
\centering
\begin{tabular}{clccccc}
\toprule
\multicolumn{2}{c}{{\small Method}} & 
{\small NLL($\downarrow$)} & 
{\small Accuracy($\uparrow$)} & 
{\small ECE($\downarrow$)} & 
{\small cNLL / cA / cECE}
& {\small \# Parameters} \\
\midrule
\multicolumn{2}{c}{Deterministic} & 0.159 & 96.0 & 0.023
& 1.05 / 76.1 / 0.153 
& 36.5M \\
\multicolumn{2}{c}{BatchEnsemble} & 0.143 & 96.2 &  0.020
& 1.02 / 77.5 / 0.129
& 36.6M \\
\multicolumn{2}{c}{MC Dropout} & 0.160 & 95.9 & 0.024
& 1.27 / 68.8 / 0.166 
& 36.5M \\
\multicolumn{2}{c}{MFVI BNN} & 0.214 & 94.7 & 0.029
& 1.46 / 71.3 / 0.181
& 73M \\
\midrule
\multirow{2}{*}{\textbf{Rank-1 BNN}} 
& Gaussian 
& 0.128 & 96.3 & \textbf{0.008} & 0.84
/ 76.7 / \textbf{0.080}
& 36.6M \\
& $\text{Cauchy}^{\text{(4 samples)}}$ 
& \textbf{0.120} & \textbf{96.5} & 0.009 & \textbf{0.74
/ 80.5} / 0.090
& 36.6M \\
\midrule
\multirow{2}{*}{Deep Ensembles} & WRN-28-5  & 0.115 & 96.3 & \textbf{0.008}  & 0.84 / 77.2  / 0.089  & 36.68M \\
& WRN-28-10 & \textbf{0.114} & \textbf{96.6}  & 0.010   & 0.81 / 77.9   / 0.087   & 146M \\
\bottomrule
\end{tabular}
\vspace{-1.5ex}
\captionof{table}{
Results for Wide ResNet-28-10 on CIFAR-10, averaged over 10 seeds.
Gaussian rank-1 BNNs with 1 sample
reach top accuracy with BatchEnsemble and otherwise
outperform baselines with comparable parameter count across all metrics.}
\label{table:cifar10}
\end{minipage}

\vspace{0.5ex}
\begin{minipage}{1.0\textwidth}
\centering
\begin{tabular}{clccccc}
\toprule
\multicolumn{2}{c}{{\small Method}} & 
{\small NLL($\downarrow$)} & 
{\small Accuracy($\uparrow$)} & 
{\small ECE($\downarrow$)} & 
{\small cNLL / cA / cECE}
& {\small \# Parameters} \\
\midrule
\multicolumn{2}{c}{Deterministic} & 0.875 & 79.8 & 0.085
& 2.70 / 51.3 / 0.239 & 36.5M \\
\multicolumn{2}{c}{BatchEnsemble} & 0.734 & 81.5 & 0.033
& 2.49 / 54.1 / 0.191 & 36.6M \\
\multicolumn{2}{c}{MC Dropout} & 0.830 & 79.6 & 0.050
& 2.33 / 51.5 / 0.148 & 36.5M \\
\multicolumn{2}{c}{MFVI BNN} & 1.030 &  77.3 & 0.111 & 3.48 / 48.0 / 0.299 & 73M \\
\midrule
\multirow{2}{*}{\textbf{Rank-1 BNN}}
& Gaussian 
& 0.692 & 81.3 &  0.018
 & 2.24 / 53.8 / \textbf{0.117} & 36.6M \\
& $\text{Cauchy}^{\text{(4 samples)}}$ 
& \textbf{0.689} & \textbf{82.4} &  \textbf{0.012}
& \textbf{2.04 / 57.8} / 0.142 & 36.6M \\
\midrule
\multirow{2}{*}{Deep Ensembles} & WRN-28-5 & 0.694  & 81.5  & 0.017 & 2.19 / 53.7  / \textbf{0.111} & 36.68M \\
& WRN-28-10   & \textbf{0.666}  & \textbf{82.7}   & 0.021  & 2.27   / 54.1    / 0.138   & 146M \\
\bottomrule
\end{tabular}
\vspace{-1.5ex}
\captionof{table}{
Results for Wide ResNet-28-10 on CIFAR-100, averaged over 10 seeds.
Gaussian rank-1 BNNs with 1 sample
reach slightly worse accuracy than BatchEnsemble and otherwise
outperform baselines with comparable parameter count.
}
\label{table:cifar100}
\end{minipage}

\vspace{0.5ex}
\begin{minipage}{1.0\textwidth}
\centering
\begin{tabular}{c l c c c | c c c}
\toprule
&
& \multicolumn{3}{c}{\small Validation} & \multicolumn{3}{c}{\small Test} 
\\
\multicolumn{2}{c}{{\small Method}} & 
{\small NLL($\downarrow$)} &
{\small \acrshort*{aucpr}($\uparrow$)} & 
{\small \acrshort*{ece}($\downarrow$)} & 
{\small NLL($\downarrow$)} & 
{\small \acrshort*{aucpr}($\uparrow$)} & 
{\small \acrshort*{ece}($\downarrow$)} 
\\

\midrule
\multicolumn{2}{c}{Deterministic}
& 0.211 
& 0.446 
& 0.0160 
& 0.213 
& 0.390 
& 0.0135 
\\

\multicolumn{2}{c}{BatchEnsemble}
& 0.215 
& 0.447 
& 0.0171 
& 0.215 
& \textbf{0.391} 
& 0.0162 
\\

\multicolumn{2}{c}{Bayesian Embeddings}
& 0.213 
& 0.449 
& 0.0193 
& 0.212 
& \textbf{0.391} 
& 0.0160 
\\

\multicolumn{2}{c}{Fully-Bayesian}
& 0.220 
& 0.424 
& 0.0162 
& 0.221 
& 0.373 
& 0.0161 
\\

\midrule

\multirow{2}{*}{\textbf{Rank-1 BNN}}
& Gaussian
& 0.209 
& \textbf{0.451} 
& 0.0156 
& \textbf{0.209} 
& \textbf{0.391} 
& 0.0132 
\\
& Cauchy
& \textbf{0.207} 
& 0.446 
& \textbf{0.0148} 
& 0.211 
& 0.383 
& \textbf{0.0130} 
\\

\midrule

\multirow{1}{*}{Deep Ensembles}
& Deterministic
& \textbf{0.202} 
& \textbf{0.453} 
& \textbf{0.0132} 
& \textbf{0.206} 
& \textbf{0.396} 
& \textbf{0.0103} 
\\
\bottomrule
\end{tabular}
\vspace{-1.5ex}
\captionof{table}{
Results for RNNs on the MIMIC-III EHR mortality task, averaged over 25 seeds, and over 25 weight samples per seed for all Bayesian models. Rank-1 Bayesian RNNs achieve the best metric performance compared to baselines.
}
\label{table:mimic3}
\vspace{-6ex}
\end{minipage}
\end{figure*}

\vspace{-1ex}
\subsection{CIFAR-100 and CIFAR-100-C}
\Cref{table:cifar100} contains NLL, accuracy, and ECE on both CIFAR-100 and CIFAR-100-C.
Rank-1 BNNs with mixture of Cauchy priors and variational posteriors outperform BatchEnsemble and similarly-sized deep ensembles by a significant margin across all metrics. To the best of our knowledge, this is the first convincing empirical success of Cauchy priors in BNNs, as it significantly improves on predictive performance, robustness, and uncertainty calibration, as observed in \Cref{fig:cifar100c} and
\Cref{sec:cifar100-c}. On the other hand, the Gaussian rank-1 BNNs have a slightly worse accuracy than BatchEnsemble, but outperform all baselines on NLL and ECE while generalizing better on CIFAR-100-C.

This is an exciting result for heavy-tailed priors in Bayesian deep learning. It has long been conjectured that such priors can be more robust to out-of-distribution data while inducing sparsity \citep{louizos2017bayesian} at the expense of accuracy. However, in both experiments summarized in \Cref{table:cifar100} and \Cref{table:cifar10} we can see significant improvements, without a compromise, on modern Wide ResNet architectures.

Rank-1 BNNs also outperform deep ensembles of WRN-28-10 models on uncertainty calibration and robustness while having 4 times fewer parameters. Rank-1 BNNs also significantly close the gap between BatchEnsemble and deep ensembles on in-distribution accuracy. Holding the number of parameters constant, rank-1 BNNs outperform deep ensembles by a significant margin across all metrics.
Conclusions compared to SWAG and Subspace Inference are consistent with CIFAR-10's.

\vspace{-1ex}
\subsection{MIMIC-III Mortality Prediction From EHRs}
Extending beyond image classification tasks, we also show results using rank-1 sequential models.  Following \citet{dusenberry2019analyzing}, we experiment with RNN models for predicting medical outcomes for patients given their de-identified electronic medical records.  More specifically, we replicate their setup for the MIMIC-III \citep{johnson2016} binary mortality task. In our case, we replace the existing variational LSTM \cite{schmidhuber1997} and affine layers with their rank-1 counterparts, and keep the variational embedding vectors.
We use global mixture distributions for the rank-1 layers, and the resulting model is a mixture model with shared stochastic embeddings. 

\Cref{table:mimic3} shows results for NLL, AUC-PR, and ECE on the validation and test sets.
We evaluate on 25 Monte Carlo samples at evaluation time versus 200 samples in the previous work, and report mean results over 25 random seeds.
Our rank-1 Bayesian RNN outperforms all other baselines, including the fully-Bayesian RNN, across all metrics.
These results demonstrate that
our rank-1 BNN methodology can be easily adapted to different types of tasks, different data modalities, and different architectures.

While Gaussian rank-1 RNNs outperform all baselines, the Cauchy variant does not perform as well in terms of AUC-PR, while still improving on NLL and ECE. This result, in addition to that of the ImageNet experiments, indicates the need for further inspection of heavy-tailed distributions in deep or recurrent architectures. In fact, ResNet-50 is a deeper architecture than WRN-28-10, while MIMIC-III RNNs can be unrolled over hundreds of time steps. Given that heavy-tailed posteriors lead to more frequent samples further away from the mode, we hypothesize that instability in the training dynamics is the main reason for underfitting.

\section{Related Work}

\textbf{Hierarchical priors and variational approximations.}
Rank-1 factors can be interpreted as scale factors that are shared across
weight elements.
\Cref{sub:hierarchical} details this and differences from other hierarchical priors
\citep{louizos2017bayesian,ghosh2017model}.
The outer product of rank-1 vectors resembles matrixvariate Gaussians \citep{louizos2016structured}: the major difference is that rank-1 priors are uncertain about the scale factors shared across rows and columns rather than fixing a covariance.
Rank-1 BNNs' variational approximation can be seen as a form of hierarchical variational model \citep{ranganath2016hierarchical} similar to multiplicative normalizing flows, which
posit an auxiliary distribution on the hidden units \citep{louizos2017multiplicative}.
In terms of the specific distribution, instead of normalizing flows we focus on mixtures, a well-known approach for expressive variational inference \citep{jaakkola1998improving,lawrence2001variational}.
Building on these classic works, we examine mixtures in ways that bridge algorithmic differences from deep ensembles and using modern model architectures.

\textbf{Variance reduction techniques for variational BNNs.}
Sampling with rank-1 factors (\Cref{eq:vectorization}) is closely related to Gaussian local reparameterization \citep{kingma2015variational,molchanov2017variational}, where
noise is reparameterized to act on the hidden units to enable weight sampling per-example, providing significant variance reduction over naively sampling a single set of weights and sharing it across the minibatch. Unlike Gaussian local reparameterization,
rank-1 factors are not limited to feedforward layers and location-scale distributions: it is exact for convolutions and recurrence and for arbitrary distributions. This is similar to ``correlated weight noise,'' which \citet{kingma2015variational} also studies and finds performs better than being fully Bayesian. Enabling weight sampling to these settings otherwise necessitates techniques such as Flipout \citep{wen2018flipout}.

\textbf{Parameter-efficient ensembles.}
Monte Carlo Dropout is arguably the most popular efficient ensembling technique, based on Bernoulli noise that deactivates hidden units during training and testing \citep{srivastava2014dropout,gal2015dropout}.
More recently, BatchEnsemble has emerged as an effective technique that is algorithmically similar to deep ensembles, but on rank-1 factors \citep{wen2019batchensemble}. We compare to both MC-dropout and BatchEnsemble as our primary baselines. If a single set of weights is sufficient (as opposed to a distribution for model uncertainty), there are also empirically successful averaging techniques such as Polyak-Ruppert \citep{ruppert1988efficient}, checkpointing, and stochastic weight averaging \citep{izmailov2018averaging}.

\textbf{Scaling up BNNs.}
We are aware of three previous works scaling up BNNs to ImageNet.
Variational Online Gauss Newton reports results on ResNet-18, outperforming a deterministic baseline in terms of NLL but not accuracy, and using 2x the number of neural network weights \citep{osawa2019practical}.
Cyclical SGMCMC~\citep{zhang2019cyclical} and adaptive thermostat MC~\citep{heek2019bayesian} report results on ResNet-50, outperforming a deterministic baseline in terms of NLL and accuracy, using at least 9 samples (i.e., 9x cost).
In our experiments, we use ResNet-50 with comparable parameter count for all methods; we examine not only NLL and accuracy, but also uncertainties via calibration and out-of-distribution evaluation;
and rank-1 BNNs do not apply strategies such as fixed KL scaling or tempering, which complicate the Bayesian interpretation.

Like rank-1 BNNs, \citet{izmailov2019subspace} perform Bayesian inference in a low-dimensional space. Instead of end-to-end training like rank-1 BNNs, it uses two stages where one first performs stochastic weight averaging and then applies PCA to form a projection matrix from the set of weights to, e.g., 5 dimensions, over which one can then perform inference. This projection matrix requires 5x the number of weights.

\section{Discussion}
\label{sec:discussion}
We described rank-1 BNNs, which posit a prior distribution over a rank-1 factor of each weight matrix and are trained with mixture variational distributions.
Rank-1 BNNs are parameter-efficient and scalable as Bayesian inference is done over a much smaller dimensionality.
Across ImageNet, CIFAR-10, CIFAR-100, and MIMIC-III, rank-1 BNNs achieve the best results on predictive and uncertainty metrics across in- and out-of-distribution data.

\section*{Acknowledgements}
We thank Ben Poole, Durk Kingma, Kevin Murphy, Tim Salimans, and Jonas Kemp for their feedback.

\bibliographystyle{icml2020}
\bibliography{refs}

\clearpage
\appendix
\onecolumn

\setcounter{theorem}{0}
\section{Variance Structure of the Rank-1 Perturbations}
\label{appnd:thm}
We hereby study how variance in the score function is captured by the full-rank weight matrix $\mathbf{W}$ parameterization versus the rank-1 $\mathbf{W}_*\circ \mathbf{r}\mathbf{s}^\rT$ parameterization.
We first note that around a local optimum $\mathbf{W}_*$, the score function $\sum_{n=1}^N f(\mathbf{x}_n|\mathbf{W})$ can be approximated using the Hessian $\sum_{n=1}^N \nabla_\mathbf{W}^2 f(\mathbf{x}_n|\mathbf{W})$:
\begin{align*}
\sum_{n=1}^N \lrp{ f(\mathbf{x}_n|\mathbf{W}) - f(\mathbf{x}_n|\mathbf{W}_*) }
\approx \frac{1}{2} \sum_{n=1}^N \sum_{h=1}^H \lrw{\mathbf{W}^{(h)}-\mathbf{W}^{(h)}_*, 
\nabla_{\mathbf{W}^{(h)}}^2 f(\mathbf{x}_n|\mathbf{W}_*) \lrp{\mathbf{W}^{(h)}-\mathbf{W}^{(h)}_*} }_F.
\end{align*}
We can therefore characterize variance around a local optimum via expected fluctuation in the score function, $\sum_{n=1}^N \E{f(\mathbf{x}_n|\mathbf{W}) - f(\mathbf{x}_n|\mathbf{W}_*)}$.
We compare here the effect of the two parameterizations: $\sum_{n=1}^N \Ep{\mathbf{W}}{f(\mathbf{x}_n|\mathbf{W}) - f(\mathbf{x}_n|\mathbf{W}_*)}$ versus $\sum_{n=1}^N \Ep{s}{f(\mathbf{x}_n|\mathbf{W}_*\circ \mathbf{r}\mathbf{s}^\rT) - f(\mathbf{x}_n|\mathbf{W}_*)}$.

In what follows, we take fully connected networks to demonstrate that the rank-1 parameterization can have the same local variance structure as the full-rank parameterization.
We first formulate the fully connected neural network in the following recursive relation.
For fully connected network of width $M$ and depth $H$, the score function $f(\mathbf{x}|\mathbf{W})$ can be recursively defined as:
\begin{align*}
&\mathbf{x}^{(0)} = \mathbf{x}, \\
&\mathbf{x}^{(h)} = \sqrt{\frac{c_\sigma}{M}} \sigma\lrp{\mathbf{W}^{(h)} \mathbf{x}^{(h-1)}}, \quad 1\leq h\leq H \\
&f(\mathbf{x}|\mathbf{W}) = a^\rT \mathbf{x}^{(H)}.
\end{align*}

\begin{theorem}[Formal]
\label{theorem:thm}
For a fully connected network of width $M$ and depth $H$ learned over $N$ data points, let $\mathbf{W}_*$ denote local minimum of $\sum_{n=1}^N f(\mathbf{x}_n|\mathbf{W})$ in the space of weight matrices.
Consider both full-rank perturbation $\lrp{\mathbf{W}-\mathbf{W}_*}$ and rank-$1$ perturbation $\lrp{\mathbf{W}_*\circ \mathbf{r}\mathbf{s}^\rT-\mathbf{W}_*}$.
Assume that the full-rank perturbation has the multiplicative covariance structure that 
\begin{align}
\Ep{\mathbf{W}^{h}}{\lrp{\mathbf{W}^{(h)}-\mathbf{W}^{(h)}_*}_{i,j} \lrp{\mathbf{W}^{(h)}-\mathbf{W}^{(h)}_*}_{k,l}}
= {\mathbf{W}_*}^{(h)}_{i,j} \mathbf{\Sigma}_{j,k} {\mathbf{W}_*}^{(h)}_{k,l},
\label{eq:assumption_noise}
\end{align}
for some symmetric positive semi-definite matrix $\mathbf{\Sigma}$.
Let $\mathbf{s}_*^{(h)}$ denote a column vector of ones.
Then if the rank-$1$ perturbation has covariance $\Ep{\mathbf{s}^{(h)}}{ \lrw{\lrp{\mathbf{s}^{(h)}-\mathbf{s}^{(h)}_*} \lrp{\mathbf{s}^{(h)}-\mathbf{s}^{(h)}_*}}^\rT }
= \mathbf{\Sigma}$, 
\begin{align}
\lefteqn{ \sum_{n=1}^N \sum_{h=1}^H \Ep{\mathbf{W}^{h}}{ \lrw{\mathbf{W}^{(h)}-\mathbf{W}^{(h)}_*, 
\nabla_{\mathbf{W}^{(h)}}^2 f(\mathbf{x}_n|\mathbf{W}_*) \lrp{\mathbf{W}^{(h)}-\mathbf{W}^{(h)}_*} }_F } } \nonumber\\
&= \sum_{n=1}^N \sum_{h=1}^H \Ep{\mathbf{s}^{(h)}}{ \lrw{\lrp{\mathbf{s}^{(h)}-\mathbf{s}^{(h)}_*}, \nabla_{\mathbf{s}^{(h)}}^2 f(\mathbf{x}_n|\mathbf{W}) \lrp{\mathbf{s}^{(h)}-\mathbf{s}^{(h)}_*}} }. 
\end{align}
\end{theorem}
Theorem~\ref{theorem:thm} demonstrates a correspondence between the covariance structure in the perturbation of $\mathbf{W}$ and that of $s$.
Since $\mathbf{\Sigma}$ can be any symmetric positive semi-definite matrix, we have demonstrated here that our rank-1 parameterization can efficiently encode a wide range of fluctuations in $\mathbf{W}$.
In particular, it is especially suited for multiplicative noise as advertised.
If the covariance of $\lrp{\mathbf{W}^{(h)}-\mathbf{W}^{(h)}_*}$ is proportional to $\mathbf{W}_* \otimes \mathbf{W}_*^\rT$ itself,
then we can simply take the covariance of $\lrp{s-\mathbf{s}_*}$ to be identity.

We devote the rest of this section to prove Theorem~\ref{theorem:thm}.
\begin{proof}[Proof of Theorem~\ref{theorem:thm}]
We first state the following lemma for the fluctuations of the score function $f$ in $\mathbf{W}$ and $s$ spaces.
\begin{lemma}
\label{lemma:perturbations}
For a fully connected network of width $M$ and depth $H$ learned over $N$ data points, let $\mathbf{W}_*$ denote local minimum of $\sum_{n=1}^N f(\mathbf{x}_n|\mathbf{W})$ in the space of weight matrices.
Then the local fluctuations of the score function in the space of the weight matrix $\mathbf{W}$ is:
\begin{align}
\lefteqn{ \Ep{\mathbf{W}^{(h)}}{ \lrw{\mathbf{W}^{(h)}-\mathbf{W}^{(h)}_*, \nabla_{\mathbf{W}^{(h)}}^2 f(\mathbf{x}_n|\mathbf{W}) \lrp{\mathbf{W}^{(h)}-\mathbf{W}^{(h)}_*}}_F } } \nonumber\\
&= \lrp{\frac{c_\sigma}{M}}^{\frac{H-h+1}{2}} 
\mathrm{trace} \Bigg( \Ep{\mathbf{W}^{(h)}}{ \lrp{\mathbf{W}^{(h)}-\mathbf{W}^{(h)}_*} \mathbf{x}^{(h-1)}_n \lrp{\mathbf{x}^{(h-1)}_n}^\rT \lrp{\mathbf{W}^{(h)}-\mathbf{W}^{(h)}_*}^\rT } \nonumber\\ 
&\cdot \mathrm{diag} \lrp{ \prod_{\mathfrak{h}=h+1}^{H} \mathrm{diag}\lrp{ \sigma'\lrp{\mathbf{W}^{(\mathfrak{h})} \mathbf{x}^{(\mathfrak{h}-1)}} } \mathbf{W}^{(\mathfrak{h})} a} 
\mathrm{diag}\lrp{ \sigma''\lrp{\mathbf{W}^{(h)} \mathbf{x}^{(h-1)}} } \Bigg).
\label{eq:perturbation_W}
\end{align}
and in the space of the low rank representation $s$,
\begin{align}
\lefteqn{ \Ep{\mathbf{s}^{(h)}}{ \lrw{\lrp{\mathbf{s}^{(h)}-\mathbf{s}^{(h)}_*}, \nabla_{\mathbf{s}^{(h)}}^2 f(\mathbf{x}_n|\mathbf{W}) \lrp{\mathbf{s}^{(h)}-\mathbf{s}^{(h)}_*}} } } \nonumber\\
&= \lrp{\frac{c_\sigma}{M}}^{\frac{H-h+1}{2}} \mathrm{trace} \Bigg( 
\mathbf{W}_*^{(h)} \E{ \mathrm{diag}\lrp{\mathbf{s}^{(h)}-\mathbf{s}^{(h)}_*} 
\lrp{\mathbf{x}_n^{(h-1)}} \lrp{\mathbf{x}_n^{(h-1)}}^\rT 
\mathrm{diag}\lrp{\mathbf{s}^{(h)}-\mathbf{s}^{(h)}_*} } \lrp{ \mathbf{W}_*^{(h)} }^\rT \nonumber\\
&\mathrm{diag} \lrp{ \prod_{\mathfrak{h}=h+1}^{H} \mathrm{diag} \lrp{ \sigma'\lrp{\mathbf{W}^{(\mathfrak{h})} \mathbf{x}^{(\mathfrak{h}-1)}} } \cdot \mathbf{W}^{\mathfrak{h}} a } 
\cdot \mathrm{diag}\lrp{ \sigma''\lrp{\mathbf{W}^{(h)} \mathbf{x}^{(h-1)}_n} } 
\Bigg).
\label{eq:perturbation_s}
\end{align}
\end{lemma}

For perturbations $\lrp{\mathbf{W}-\mathbf{W}_*}$ with a multiplicative structure, we can write that
\begin{align*}
\Ep{\mathbf{W}^{h}}{\lrp{\mathbf{W}^{(h)}-\mathbf{W}^{(h)}_*}_{i,j} \lrp{\mathbf{W}^{(h)}-\mathbf{W}^{(h)}_*}_{k,l}}
= {\mathbf{W}_*}_{i,j} \mathbf{\Sigma}_{j,k} {\mathbf{W}_*}_{k,l},
\end{align*}
for some matrix $\mathbf{\Sigma}$ (in the simplest case where $\mathbf{\Sigma} = \epsilon \cdot \mI$, this corresponds to the covariance of $\lrp{\mathbf{W}-\mathbf{W}_*}$ being a decomposable tensor: $\Ep{\mathbf{W}^{h}}{\lrp{\mathbf{W}^{(h)}-\mathbf{W}^{(h)}_*} \lrp{\mathbf{W}^{(h)}-\mathbf{W}^{(h)}_*}} = \epsilon \cdot \mathbf{W}_* \otimes \mathbf{W}_*^\rT$).
In this multiplicative perturbation case, we can show that if $\Ep{\mathbf{s}^{(h)}}{ \lrp{\mathbf{s}^{(h)}-\mathbf{s}_*} \lrp{\mathbf{s}^{(h)}-\mathbf{s}_*}^\rT } = \mathbf{\Sigma}$, then
\begin{align*}
\lefteqn{ \Ep{\mathbf{W}^{(h)}}{ \lrp{\mathbf{W}^{(h)}-\mathbf{W}^{(h)}_*} \mathbf{x}^{(h-1)}_n \lrp{\mathbf{x}^{(h-1)}_n}^\rT \lrp{\mathbf{W}^{(h)}-\mathbf{W}^{(h)}_*}^\rT } } \\
&= \mathbf{W}_*^{(h)} \mathrm{diag}\lrp{\mathbf{x}_n^{(h-1)}} \mathbf{\Sigma} \ \mathrm{diag}\lrp{\mathbf{x}_n^{(h-1)}} \lrp{ \mathbf{W}_*^{(h)} }^\rT \\
&= \mathbf{W}_*^{(h)} \mathrm{diag}\lrp{\mathbf{x}_n^{(h-1)}}
\Ep{\mathbf{s}^{(h)}}{ \lrp{\mathbf{s}^{(h)}-\mathbf{s}_*} \lrp{\mathbf{s}^{(h)}-\mathbf{s}_*}^\rT }
\mathrm{diag}\lrp{\mathbf{x}_n^{(h-1)}} 
\lrp{ \mathbf{W}_*^{(h)} }^\rT \\
&= \mathbf{W}_*^{(h)} \Ep{\mathbf{s}^{(h)}}{ \mathrm{diag}\lrp{\mathbf{s}^{(h)}-\mathbf{s}_*} 
\lrp{\mathbf{x}_n^{(h-1)}} \lrp{\mathbf{x}_n^{(h-1)}}^\rT 
\mathrm{diag}\lrp{\mathbf{s}^{(h)}-\mathbf{s}_*} } \lrp{ \mathbf{W}_*^{(h)} }^\rT.
\end{align*}
Plugging this result into equations~\ref{eq:perturbation_W} and~\ref{eq:perturbation_s}, we know that for any $n$ and $h$,
\begin{align*}
\Ep{\mathbf{W}^{(h)}}{ \lrw{\mathbf{W}^{(h)}-\mathbf{W}^{(h)}_*, \nabla_{\mathbf{W}^{(h)}}^2 f(\mathbf{x}_n|\mathbf{W}) \lrp{\mathbf{W}^{(h)}-\mathbf{W}^{(h)}_*}}_F } 
= \Ep{\mathbf{s}^{(h)}}{ \lrw{\lrp{\mathbf{s}^{(h)}-\mathbf{s}^{(h)}_*}, \nabla_{\mathbf{s}^{(h)}}^2 f(\mathbf{x}_n|\mathbf{W}) \lrp{\mathbf{s}^{(h)}-\mathbf{s}^{(h)}_*}} }.
\end{align*}
Therefore, 
\begin{align}
\lefteqn{ \sum_{n=1}^N \sum_{h=1}^H \Ep{\mathbf{W}^{h}}{ \lrw{\mathbf{W}^{(h)}-\mathbf{W}^{(h)}_*, 
\nabla_{\mathbf{W}^{(h)}}^2 f(\mathbf{x}_n|\mathbf{W}_*) \lrp{\mathbf{W}^{(h)}-\mathbf{W}^{(h)}_*} }_F } } \nonumber\\
&= \sum_{n=1}^N \sum_{h=1}^H \Ep{\mathbf{s}^{(h)}}{ \lrw{\lrp{\mathbf{s}^{(h)}-\mathbf{s}^{(h)}_*}, \nabla_{\mathbf{s}^{(h)}}^2 f(\mathbf{x}_n|\mathbf{W}) \lrp{\mathbf{s}^{(h)}-\mathbf{s}^{(h)}_*}} }. 
\end{align}
\end{proof}

\begin{proof}[Proof of Lemma~\ref{lemma:perturbations}]
We first analyze the local geometric structures of the score function in the space of the full-rank weight matrix $\mathbf{W}$ and the low rank vector $s$, respectively.
We then leverage this Hessian information to finish our proof.

\paragraph{Local Geometry of the score function $f(\mathbf{x}_n|\mathbf{W}_* \circ \mathbf{r}\mathbf{s}^\rT)$:}
We can first compute the gradient of weight $\mathbf{W}$ at $h$-th layer for the predictive score function $f$ of an $H$ layer fully connected neural network taken at data point $\mathbf{x}_n$: 
\begin{align*}
\lefteqn{ \nabla_{\mathbf{W}^{(h)}} f(\mathbf{x}_n|\mathbf{W}) } \\
&= \frac{\partial \mathbf{x}_n^{(h)}}{\partial \mathbf{W}^{(h)}} \nabla_{\mathbf{x}_n^{(h)}} f(\mathbf{x}|\mathbf{W}) \\
&= \sqrt{\frac{c_\sigma}{M}} \mathrm{diag}\lrp{ \sigma'\lrp{\mathbf{W}^{(h)} \mathbf{x}^{(h-1)}_n} } \cdot \frac{\partial}{\partial \mathbf{x}^{(h)}_n} f(\mathbf{x}_n|\mathbf{W}) \cdot \lrp{\mathbf{x}^{(h-1)}_n}^\rT \\
&= \lrp{\frac{c_\sigma}{M}}^{\frac{H-h+1}{2}} \mathrm{diag}\lrp{ \sigma'\lrp{\mathbf{W}^{(h)} \mathbf{x}^{(h-1)}_n} } \cdot \prod_{\mathfrak{h}=h+1}^{H} \mathrm{diag}\lrp{ \sigma'\lrp{\mathbf{W}^{(\mathfrak{h})} \mathbf{x}^{(\mathfrak{h}-1)}_n} } \cdot \mathbf{W}^{\mathfrak{h}} a \cdot \lrp{\mathbf{x}^{(h-1)}_n}^\rT \\
&= \lrp{\frac{c_\sigma}{M}}^{\frac{H-h+1}{2}} \underbrace{ \sigma'\lrp{\mathbf{W}^{(h)} \mathbf{x}^{(h-1)}_n} 
\prod_{\mathfrak{h}=h+1}^{H} \circ \ \sigma'\lrp{\mathbf{W}^{(\mathfrak{h})} \mathbf{x}^{(\mathfrak{h}-1)}_n} \cdot \mathbf{W}^{\mathfrak{h}} a }_{v_n^{(h)}} \cdot \lrp{\mathbf{x}^{(h-1)}_n}^\rT.
\end{align*}
If we instead take the gradient over the vector $s$, we obtain that
\begin{align*}
\lefteqn{ \nabla_{\mathbf{s}^{(h)}} f(\mathbf{x}_n|\mathbf{W}_* \circ \mathbf{r}\mathbf{s}^\rT) } \\
&= \lrw{ \frac{\partial}{\partial \mathbf{W}^{(h)}} f(\mathbf{x}_n|\mathbf{W}) , \frac{\partial \mathbf{W}^{(h)}}{\partial \mathbf{s}^{(h)}} }_F \\
&= \lrp{ \frac{\partial}{\partial \mathbf{W}^{(h)}} f(\mathbf{x}_n|\mathbf{W}) }^\rT \circ \lrp{\mathbf{W}_*^{(h)}}^\rT \mathbf{r}^{(h)} \\
&= \lrp{\frac{c_\sigma}{M}}^{\frac{H-h+1}{2}} \lrp{\mathbf{W}_*^{(h)}}^\rT \circ \mathbf{x}^{(h-1)}_n \cdot \lrp{v_n^{(h)}}^\rT \mathbf{r}^{(h)} \nonumber\\
&= \lrp{\frac{c_\sigma}{M}}^{\frac{H-h+1}{2}} \lrp{\mathbf{W}_*^{(h)}}^\rT \lrp{ \mathbf{r}^{(h)} \circ v_n^{(h)}} \circ \mathbf{x}^{(h-1)}_n \\
&= \lrp{\frac{c_\sigma}{M}}^{\frac{H-h+1}{2}} 
\mathrm{diag}\lrp{\mathbf{x}^{(h-1)}_n} \lrp{\mathbf{W}_*^{(h)}}^\rT \mathrm{diag}\lrp{\mathbf{r}^{(h)}} v_n^{(h)}.
\end{align*}

We can further analyze the Hessian of $f$:
\begin{align}
\lefteqn{ \nabla_{\mathbf{W}^{(h)}}^2 f(\mathbf{x}_n|\mathbf{W}) } \nonumber\\
&= \lrp{\frac{c_\sigma}{M}}^{\frac{H-h+1}{2}} 
\mathrm{diag}\lrp{ \prod_{\mathfrak{h}=h+1}^{H} \mathrm{diag}\lrp{ \sigma'\lrp{\mathbf{W}^{(\mathfrak{h})} \mathbf{x}^{(\mathfrak{h}-1)}_n} } \mathbf{W}^{\mathfrak{h}} a } \mathrm{diag}\lrp{ \sigma''\lrp{\mathbf{W}^{(h)} \mathbf{x}^{(h-1)}_n} } 
\otimes \mathbf{x}^{(h-1)}_n \lrp{\mathbf{x}^{(h-1)}_n}^\rT.
\label{eq:Hessian_W}
\end{align}
Whereas for $s$,
\begin{align}
\lefteqn{ \nabla_{\mathbf{s}^{(h)}}^2 f(\mathbf{x}_n|\mathbf{W}_* \circ \mathbf{r}\mathbf{s}^\rT) } \nonumber\\
&= \lrp{\frac{c_\sigma}{M}}^{\frac{H-h+1}{2}} 
\mathrm{diag}\lrp{\mathbf{x}_n^{(h-1)}} \lrp{ \mathbf{W}_*^{(h)} }^\rT \mathrm{diag}\lrp{\mathbf{r}^{(h)}} 
\mathrm{diag} \lrp{ \prod_{\mathfrak{h}=h+1}^{H} \mathrm{diag} \lrp{ \sigma'\lrp{\mathbf{W}^{(\mathfrak{h})} \mathbf{x}^{(\mathfrak{h}-1)}} } \cdot \mathbf{W}^{(\mathfrak{h})} a } \nonumber\\
&\cdot \mathrm{diag}\lrp{ \sigma''\lrp{\mathbf{W}^{(h)} \mathbf{x}^{(h-1)}_n} } 
\mathrm{diag}\lrp{\mathbf{r}^{(h)}}  \mathbf{W}_*^{(h)} \mathrm{diag}\lrp{\mathbf{x}_n^{(h-1)}}.
\label{eq:Hessian_s}
\end{align}

\paragraph{Variance Structures in the Score Function:}
Applying the results in equations~\ref{eq:Hessian_W} and~\ref{eq:Hessian_s}, we obtain that
\begin{align*}
\lefteqn{ \Ep{\mathbf{W}^{(h)}}{ \lrw{\mathbf{W}^{(h)}-\mathbf{W}^{(h)}_*, \nabla_{\mathbf{W}^{(h)}}^2 f(\mathbf{x}_n|\mathbf{W}) \lrp{\mathbf{W}^{(h)}-\mathbf{W}^{(h)}_*}}_F } } \\
&= \lrp{\frac{c_\sigma}{M}}^{\frac{H-h+1}{2}}
\mathbb{E}_{\mathbf{W}^{(h)}} \Bigg[ \lrp{\mathbf{x}^{(h-1)}_n}^\rT \lrp{\mathbf{W}^{(h)}-\mathbf{W}^{(h)}_*}^\rT \\ 
&\mathrm{diag}\lrp{ \prod_{\mathfrak{h}=h+1}^{H} \mathrm{diag}\lrp{ \sigma'\lrp{\mathbf{W}^{(\mathfrak{h})} \mathbf{x}^{(\mathfrak{h}-1)}} } \mathbf{W}^{\mathfrak{h}} a } 
\mathrm{diag}\lrp{ \sigma''\lrp{\mathbf{W}^{(h)} \mathbf{x}^{(h-1)}} } 
\lrp{\mathbf{W}^{(h)}-\mathbf{W}^{(h)}_*} \mathbf{x}^{(h-1)}_n \Bigg]\\
&= \lrp{\frac{c_\sigma}{M}}^{\frac{H-h+1}{2}} 
\mathrm{trace} \Bigg( \Ep{\mathbf{W}^{(h)}}{ \lrp{\mathbf{W}^{(h)}-\mathbf{W}^{(h)}_*} \mathbf{x}^{(h-1)}_n \lrp{\mathbf{x}^{(h-1)}_n}^\rT \lrp{\mathbf{W}^{(h)}-\mathbf{W}^{(h)}_*}^\rT } \\ 
&\cdot \mathrm{diag} \lrp{ \prod_{\mathfrak{h}=h+1}^{H} \mathrm{diag}\lrp{ \sigma'\lrp{\mathbf{W}^{(\mathfrak{h})} \mathbf{x}^{(\mathfrak{h}-1)}} } \mathbf{W}^{\mathfrak{h}} a} 
\mathrm{diag}\lrp{ \sigma''\lrp{\mathbf{W}^{(h)} \mathbf{x}^{(h-1)}} } \Bigg).
\end{align*}
and that
\begin{align*}
\lefteqn{ \Ep{\mathbf{s}^{(h)}}{ \lrw{\lrp{\mathbf{s}^{(h)}-\mathbf{s}^{(h)}_*}, \nabla_{\mathbf{s}^{(h)}}^2 f(\mathbf{x}_n|\mathbf{W}) \lrp{\mathbf{s}^{(h)}-\mathbf{s}^{(h)}_*}} } } \\
&= \lrp{\frac{c_\sigma}{M}}^{\frac{H-h+1}{2}} 
\mathbb{E} \lrp{\mathbf{W}_*^{(h)} \lrp{\mathbf{x}_n^{(h-1)} \circ \lrp{\mathbf{s}^{(h)}-\mathbf{s}^{(h)}_*}} \circ \mathbf{r}^{(h)}_*}^\rT
\mathrm{diag} \lrp{ \prod_{\mathfrak{h}=h+1}^{H} \mathrm{diag} \lrp{ \sigma'\lrp{\mathbf{W}^{(\mathfrak{h})} \mathbf{x}^{(\mathfrak{h}-1)}} } \cdot \mathbf{W}^{\mathfrak{h}} a } \\
&\cdot \mathrm{diag}\lrp{ \sigma''\lrp{\mathbf{W}^{(h)} \mathbf{x}^{(h-1)}_n} } 
\cdot \mathbf{W}_*^{(h)} \lrp{\mathbf{x}_n^{(h-1)} \circ \lrp{\mathbf{s}^{(h)}-\mathbf{s}^{(h)}_*}} \circ \mathbf{r}^{(h)}_* \\
&= \lrp{\frac{c_\sigma}{M}}^{\frac{H-h+1}{2}} \mathrm{trace} \Bigg( 
\mathbf{W}_*^{(h)} \E{ \mathrm{diag}\lrp{\mathbf{s}^{(h)}-\mathbf{s}^{(h)}_*} 
\lrp{\mathbf{x}_n^{(h-1)}} \lrp{\mathbf{x}_n^{(h-1)}}^\rT 
\mathrm{diag}\lrp{\mathbf{s}^{(h)}-\mathbf{s}^{(h)}_*} } \lrp{ \mathbf{W}_*^{(h)} }^\rT \\
&\mathrm{diag} \lrp{ \prod_{\mathfrak{h}=h+1}^{H} \mathrm{diag} \lrp{ \sigma'\lrp{\mathbf{W}^{(\mathfrak{h})} \mathbf{x}^{(\mathfrak{h}-1)}} } \cdot \mathbf{W}^{\mathfrak{h}} a } 
\cdot \mathrm{diag}\lrp{ \sigma''\lrp{\mathbf{W}^{(h)} \mathbf{x}^{(h-1)}_n} } 
\Bigg).
\end{align*}

\end{proof}

\clearpage

\section{Additional Experimental Details and Hyperparameters}
\label{appnd:hparams}

We experiment with both mixture of Gaussian and mixture of Cauchy priors (and variational posteriors) for the rank-1 factors. All reported results are averages over 10 runs for the image classification tasks and 25 runs for the EHR task. For Gaussian distributions in the image tasks, we achieve superior metric performance using only 1 Monte Carlo sample for each of 4 components to estimate the integral in \Cref{eq:loss} for both training and evaluation, unlike much of the BNN literature, and we show further gains from using larger numbers of samples (4 and 25; see \cref{sec:num-weight-samples}). For Cauchy distributions on those image tasks, we use 1 Monte Carlo sample for each of 4 components for training, and use 4 samples per component during evaluation. 
For the EHR task, we also use only 1 sample during training, but use 25 samples during evaluation (down from 200 samples for the Bayesian models in \citet{dusenberry2019analyzing}). See \Cref{appnd:hparams} for details on hyperparameters. Our code uses TensorFlow and Edward2's Bayesian Layers \citep{tran2018bayesian}; all experiments are available at
\url{https://github.com/google/edward2}.

For rank-1 BNNs, there are three hyperparameters in addition to the deterministic baseline's: the number of mixture components (we fix it at 4); prior standard deviation (we vary among 0.05, 0.1, and 1); and the mean initialization for variational posteriors (either random sign flips with probability random\_sign\_init or a random normal with mean 1 and standard deviation random\_sign\_init).
All hyperparameters for our rank-1 BNNs can be found in Tables \ref{table:cifar10-hparams}, \ref{table:imagenet-hparams}, and \ref{table:mimic-hparams}.

Following \Cref{sec:rank-1}'s ablations, we always (with one exception) use a prior with mean at 1, the average per-component log-likelihood, and initialize variational posterior standard deviations under the dropout parameterization as $10^{-3}$ for Gaussian priors and $10^{}$. The one exception is the Cauchy rank-1 Bayesian RNN on MIMIC-III, where we use a prior with mean 0.5.

Rank-1 BNNs apply rank-1 factors to all layers in the network except for normalization layers and the embedding layers in the MIMIC-III models.
We are not Bayesian about the biases, but we do not find it made a difference.

We use a linear KL annealing schedule for 2/3 of the total number of training epochs (we also tried 1/3 and 1/4 and did not find the setting sensitive). Rank-1 BNNs use 250 training epochs for CIFAR-10/100 (deterministic uses 200); 135 epochs for ImageNet (deterministic uses 90); and 12000 to 25000 steps for MIMIC-III.

All methods use the largest batch size before we see a generalization gap in any method. For ImageNet, this is 32 TPUv2 cores with a per-core batch size of 128; for CIFAR-10/100, this is 8 TPUv2 cores with a per-core batch size of 64; for MIMIC-III this differs depending on the architecture. All CIFAR-10/100 and ImageNet methods use SGD with momentum with the same step-wise learning rate decay schedule, built on the deterministic baseline. For MIMIC-III, we use Adam \cite{kingma2014adam} with no decay schedule.

For MIMIC-III, all hyperparameters for the baselines match those of \citet{dusenberry2019analyzing}, except we used a batch size of 128 for the deterministic and Bayesian Embeddings models. Since \citet{dusenberry2019analyzing} tuned each model separately, including the architecture sizes, we also tuned our rank-1 Bayesian RNN architecture sizes (for performance and memory constraints). Of note, the Gaussian rank-1 RNN has a slightly smaller architecture (rnn\_dim=512 vs. 1024).

\begin{minipage}{1.0\textwidth}
\centering
\begin{tabular}{|c *{4}{|c|}}
\hline
\textbf{Dataset} & \multicolumn{2}{c|}{CIFAR-10} &  \multicolumn{2}{c|}{CIFAR-100} \\
\midrule
ensemble\_size & \multicolumn{4}{c|}{$4$} \\
base\_learning\_rate & \multicolumn{4}{c|}{$0.1$} \\
prior\_mean & \multicolumn{4}{c|}{$1.0$} \\
per\_core\_batch\_size & \multicolumn{4}{c|}{$64$} \\ 
num\_cores & \multicolumn{4}{c|}{$8$} \\
lr\_decay\_ratio & \multicolumn{4}{c|}{$0.2$} \\
train\_epochs & \multicolumn{4}{c|}{$250$} \\
lr\_decay\_epochs & \multicolumn{4}{c|}{[80, 160, 180]} \\
kl\_annealing\_epochs & \multicolumn{4}{c|}{$200$} \\
\midrule
l2 & \multicolumn{2}{c|}{$0.0001$} & \multicolumn{2}{c|}{$0.0003$} \\
\midrule
\textbf{Method} & Normal & Cauchy & Normal & Cauchy \\
\midrule
alpha\_initializer & trainable\_normal & trainable\_cauchy & trainable\_normal & trainable\_cauchy\\ 
alpha\_regularizer & normal\_kl\_divergence & cauchy\_kl\_divergence & normal\_kl\_divergence & cauchy\_kl\_divergence \\ 
gamma\_initializer & trainable\_normal & trainable\_cauchy & trainable\_normal & trainable\_cauchy \\ 
gamma\_regularizer & normal\_kl\_divergence & cauchy\_kl\_divergence & normal\_kl\_divergence & cauchy\_kl\_divergence \\ 
prior\_stddev & $0.1$ & $0.1$ & $0.1$ & $0.01$ \\
dropout\_rate (init) & $0.001$ & $10^{-6}$ & $0.001$ & $10^{-6}$ \\ 
random\_sign\_init & $-0.5$  & $-0.5$ & $-1.0$ & $-1.0$  \\ 
\bottomrule
\end{tabular}
\vspace{-1ex}
\captionof{table}{
Hyperparameter values for Rank-1 BNNs with Wide ResNet-28-10  on CIFAR-10 and CIFAR-100. Alpha and Gamma refer to the $r$ and $s$ vectors in the main text. The initializer determines the form of the variational posterior whereas the regularizer dictates the choice of priors. Note that all priors and approximate posteriors are mixtures.}
\label{table:cifar10-hparams}
\end{minipage}

\begin{minipage}{1.0\textwidth}
\centering
\begin{tabular}{|c *{2}{|c|}}
\hline
\textbf{Dataset} & \multicolumn{2}{c|}{ImageNet} \\
\midrule
ensemble\_size & \multicolumn{2}{c|}{$4$} \\
base\_learning\_rate & \multicolumn{2}{c|}{$0.1$} \\
prior\_mean & \multicolumn{2}{c|}{$1.0$} \\
per\_core\_batch\_size  & \multicolumn{2}{c|}{$128$} \\ 
num\_cores  & \multicolumn{2}{c|}{$32$} \\
lr\_decay\_ratio & \multicolumn{2}{c|}{$0.1$} \\
train\_epochs  & \multicolumn{2}{c|}{$135$} \\
lr\_decay\_epochs & \multicolumn{2}{c|}{[45, 90, 120]} \\
kl\_annealing\_epochs & \multicolumn{2}{c|}{$90$} \\
l2 & \multicolumn{2}{c|}{$0.0001$} \\
\midrule
\textbf{Method} & Normal & Cauchy  \\
\midrule
alpha\_initializer & trainable\_normal & trainable\_cauchy \\ 
alpha\_regularizer & normal\_kl\_divergence & cauchy\_kl\_divergence \\ 
gamma\_initializer & trainable\_normal & trainable\_cauchy \\ 
gamma\_regularizer & normal\_kl\_divergence & cauchy\_kl\_divergence \\ 
prior\_stddev & $0.05$ & $0.005$ \\
dropout\_rate (init) & $0.001$ & $10^{-6}$\\ 
random\_sign\_init & $-0.75$ & $-0.5$  \\ 
\bottomrule
\end{tabular}
\vspace{-1ex}
\captionof{table}{
Hyperparameter values for Rank-1 BNNs with ResNet-50 on ImageNet. }
\label{table:imagenet-hparams}
\end{minipage}

\begin{minipage}{1.0\textwidth}
\centering
\begin{tabular}{|c *{2}{|c|}}
\hline
\textbf{Dataset} & \multicolumn{2}{c|}{MIMIC-III} \\
\midrule
ensemble\_size & \multicolumn{2}{c|}{4} \\
embeddings\_initializer & \multicolumn{2}{c|}{trainable\_normal} \\
embeddings\_regularizer & \multicolumn{2}{c|}{normal\_kl\_divergence} \\
random\_sign\_init & \multicolumn{2}{c|}{$0.5$}\\ 
rnn\_dim & \multicolumn{2}{c|}{$512$} \\
hidden\_layer\_dim & \multicolumn{2}{c|}{$512$} \\
l2 & \multicolumn{2}{c|}{$1\mathrm{e}{-4}$} \\
bagging\_time\_precision & \multicolumn{2}{c|}{$86400$} \\
num\_ece\_bins & \multicolumn{2}{c|}{$15$} \\
\midrule
\textbf{Method} & Normal & Cauchy  \\
\midrule
alpha\_initializer & trainable\_normal & trainable\_cauchy \\ 
alpha\_regularizer & normal\_kl\_divergence & cauchy\_kl\_divergence \\ 
gamma\_initializer & trainable\_normal & trainable\_cauchy \\ 
gamma\_regularizer & normal\_kl\_divergence & cauchy\_kl\_divergence \\ 
prior\_mean & $1.$ & $0.5$ \\
prior\_stddev & $0.1$ & $0.0001$ \\
dropout\_rate (init) & $0.001$ & $5\mathrm{e}{-7}$\\ 
dense\_embedding\_dimension & $32$ & $16$ \\
embedding\_dimension\_multiplier & $0.85827$ & $0.984215$ \\
batch\_size & $128$ & $32$ \\
learning\_rate & $0.00030352$ & $0.001$ \\
fast\_weight\_lr\_multiplier & $1.$ & $0.575$ \\
kl\_annealing\_steps & $20000$ & $694216$ \\
max\_steps & $25000$ & $12000$ \\
bagging\_aggregate\_older\_than & $-1$ & $60*60*24*90$ \\
clip\_norm & $7.29199$ & $1.83987$ \\
\bottomrule
\end{tabular}
\vspace{-1ex}
\captionof{table}{
Hyperparameter values for Rank-1 Bayesian RNNs on MIMIC-III.}
\label{table:mimic-hparams}
\end{minipage}

\clearpage
\section{Further Ablation Studies}

\subsection{Initialization}
\label{subsec:init}

\begin{figure}[!t]
\centering
\begin{subfigure}[t]{0.32\columnwidth}
\centering
\includegraphics[width=0.925\columnwidth]{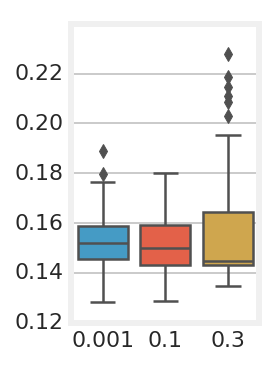}
\caption{Test NLL}
\end{subfigure}\hfill
\begin{subfigure}[t]{0.29\columnwidth}
\centering
\includegraphics[width=\columnwidth]{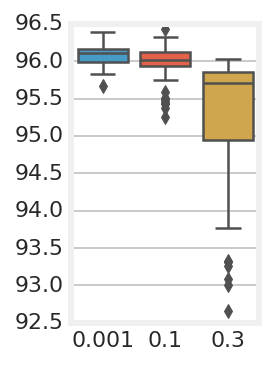}
\caption{Test Accuracy}
\end{subfigure}\hfill
\begin{subfigure}[t]{0.32\columnwidth}
\centering
\includegraphics[width=\columnwidth]{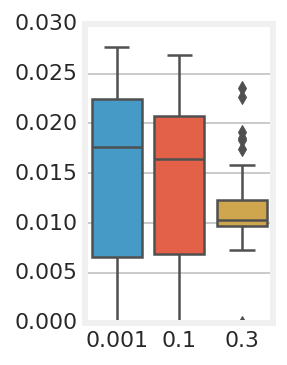}
\caption{Test ECE}
\end{subfigure}
\vspace{-1ex}
\caption{Dropout-parameterized initialization for the variational distribution's standard deviations. Each boxplot is over 96 runs from a hyperparameter sweep. Using a dropout rate (and therefore standard deviation) close to zero gets much better accuracy at a slight cost of calibration error.
}
\label{fig:initialization}
\end{figure}

There are two sets of parameters to initialize: the set of weights $\mathbf{W}$ and the variational parameters of the rank-1 distributions $q(\mathbf{r})$ and $q(\mathbf{s})$. The weights are initialized just as in deterministic networks.
For the variational posterior distributions, we initialize the mean following BatchEnsemble: random sign flips of $\pm1$ or a draw from a normal centered at 1. This encourages each sampled vector to be roughly orthogonal from one another (thus inducing different directions for diverse solutions as one takes gradient steps); unit mean encourages the identity.

For the variational standard deviation parameters $\sigma$, we explore two approaches (\Cref{fig:initialization}).
The first is a ``deterministic initialization,'' where $\sigma$ is set close to zero such that---when combined with KL annealing---the initial optimization trajectory resembles a deterministic network's.
This is commonly used for variational inference (e.g., \citet{kucukelbir2017automatic}). Though this aids optimization and aims to prevent underfitting, one potential reason for why BNNs still underperform is that a deterministic initialization encourages poorly estimated uncertainties: the distribution of weights may be less prone to expand as the
annealed KL
penalizes deviations from the prior (the cost tradeoff under the likelihood may be too high).
Alternatively, we also try a ``dropout initialization'', where standard deviations are reparameterized with a dropout rate: $\sigma=\sqrt{p / (1-p)}$ where $p$ is the binary dropout probability.\footnote{
To derive this, observe that dropout's Bernoulli noise, which takes the value $0$ with probability $p$ and $1/(1-p)$ otherwise, has mean $1$ and variance $p/(1-p)$ \citep{srivastava2014dropout}.
}
Dropout rates between 0.1 and 0.3 (common in modern architectures) imply a standard deviation of 0.3-0.65.
\Cref{fig:initialization} shows accuracy and calibration both decrease as a function of initialized dropout rate;
NLL stays roughly the same. We recommend deterministic initialization as the accuracy gains justify the minor cost in calibration.

\subsection{Real-valued Scale Parameterization}
\label{sub:real}

As shown in \Cref{eq:rank1-weight-prior}, the hierarchical prior over $\mathbf{r}$ and $\mathbf{s}$ induces a prior over the scale parameters of the layer's weights. A natural question that arises is: should the $\mathbf{r}$ and $\mathbf{s}$ priors be constrained to be positive-valued, or left unconstrained as real-valued priors? Intuitively, real-valued priors are preferable because they can modulate the sign of the layer's inputs and outputs. To determine whether this is beneficial and necessary, we perform an ablation under our CIFAR-10 setup (\Cref{sec:experiments}). In this experiment, we compare a global mixture of Gaussians for the real-valued prior, and a global mixture of log-Gaussian distributions for the positive-valued prior. For each, we tune over the initialization of the prior's standard deviation, and the L2 regularization for the point-wise estimated $\mathbf{W}$. For the Gaussians, we also tune over the initialization of the prior's mean.

\Cref{fig:real_vs_pos} displays our findings. Similar to study of priors over $\mathbf{r}$, $\mathbf{s}$, or both, we compare results across NLL, accuracy, and ECE on the test set and CIFAR-10-C corruptions dataset. We find that both setups are comparable on test accuracy, and that the real-valued setup outperforms the other on test NLL and ECE. For the corruptions task, the two setups compare equally on NLL, and differ on accuracy and ECE.

\begin{figure}[!ht]
\centering
\includegraphics[width=0.8\columnwidth]{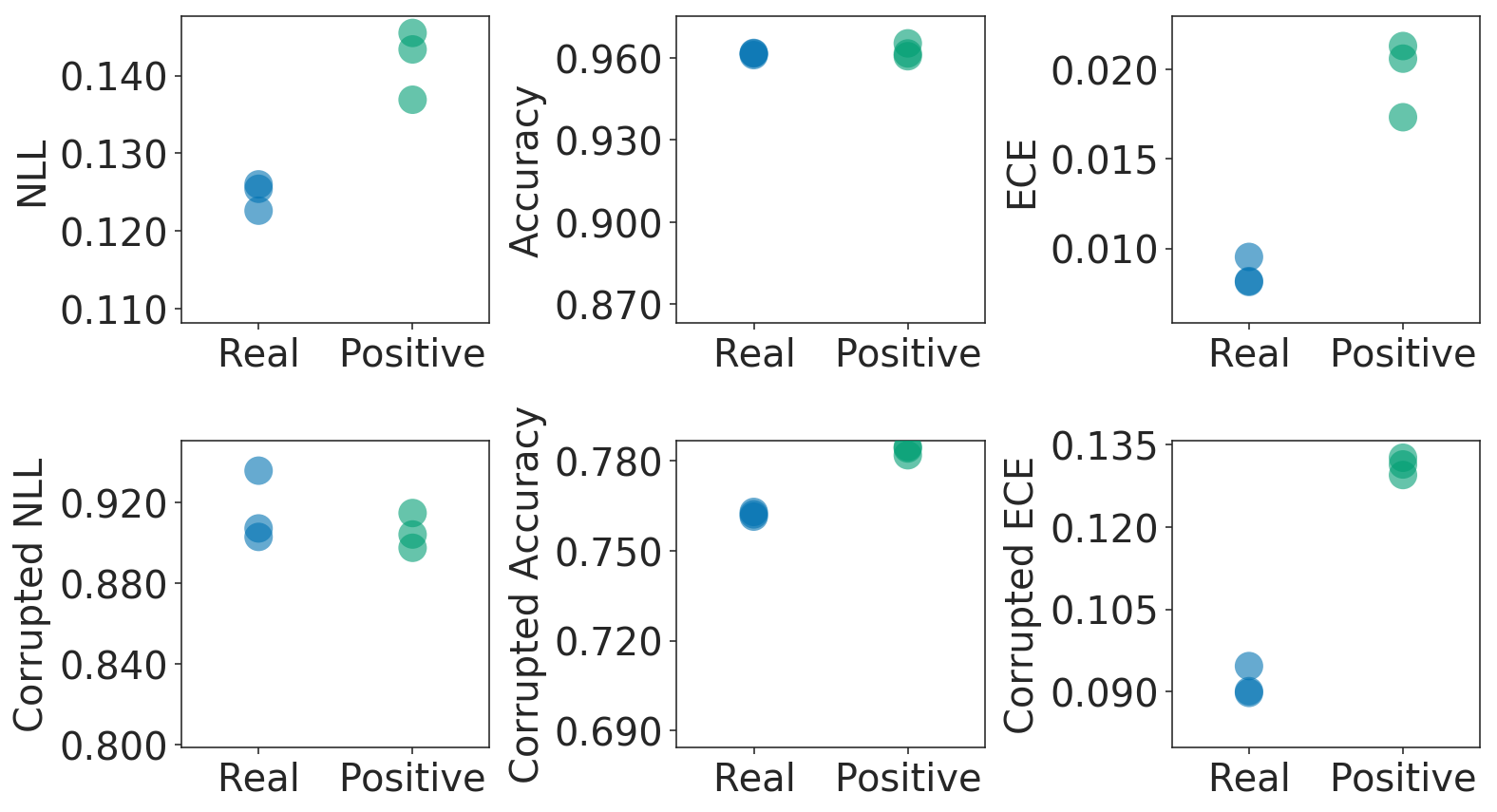}
\caption{Real-valued vs positive-valued priors over $\mathbf{s}$ and $\mathbf{r}$, each evaluated over three runs on the CIFAR-10 test set and CIFAR-10-C corrupted dataset.}
\label{fig:real_vs_pos}
\end{figure}

\subsection{Number of Evaluation Samples}
\label{sec:num-weight-samples}

In \Cref{table:cifar10-num-samples}, we experiment with using multiple weight samples, per mixture component, per example, at evaluation time for our Wide ResNet-28-10 model trained on CIFAR-10. In all cases, we use the same model that was trained using only a single weight sample (per mixture component, per example). As expected, an increased number of samples improves metric performance, with a significant improvement across all corrupted metrics. This demonstrates one of the benefits to incorporating local distributions over each mixture component, namely that given an increased computational budget, one can improve upon the metric performance at prediction time.

\begin{figure*}[!tb]
\begin{minipage}{1.0\textwidth}
\centering
\begin{tabular}{cccccc}
\\
\toprule
\multicolumn{2}{c}{Method} & NLL($\downarrow$) & Accuracy($\uparrow$) & ECE($\downarrow$) & cNLL / cA / cECE
\\% & \# Parameters \\
\midrule
\multirow{3}{*}{Rank-1 BNN - Gaussian}
& 1 sample & 0.128 & \textbf{96.3} & 0.008 & 0.84 / 76.7 / 0.080 \\
& 4 samples & 0.126 & \textbf{96.3} & 0.008 & 0.80 / 77.3 / 0.074 \\
& 25 samples & \textbf{0.125} & \textbf{96.3} & \textbf{0.007} & \textbf{0.77} / \textbf{77.8} / \textbf{0.070} \\

\midrule
\multirow{1}{*}{Rank-1 BNN - Cauchy}
& 4 samples & \textbf{0.120} & \textbf{96.5} & 0.009 & \textbf{0.74 / 80.5} / 0.090 \\

\midrule
\multirow{2}{*}{Deep Ensembles} & WRN-28-5  & 0.115 & 96.3 & 0.008  & 0.84 / 77.2  / 0.089 \\
& WRN-28-10 & \textbf{0.114} & \textbf{96.6}  & 0.010   & 0.81 / 77.9 / 0.087 \\
\bottomrule
\end{tabular}
\vspace{-1ex}
\captionof{table}{
Results across multiple weight samples (per mixture component, per example) at evaluation time for Wide ResNet-28-10 on CIFAR-10. Greater than 1 sample with Gaussian distributions yields a marginal improvement on in-distribution NLL and ECE, while yielding a significant improvement on all corrupted metrics. Cauchy rank-1 BNNs with 4 weight samples outperform Gaussians on all metrics except ECE. Note that training still uses a single weight sample (per mixture component, per example) for both Gaussian and Cauchy rank-1 BNNs. We include the deep ensembles results again to show that with an increased number of samples, a rank-1 WRN-28-10 can exceed an ensemble of WRN-28-5 models, which collectively have a comparable parameter count.
}
\label{table:cifar10-num-samples}
\vspace{-3ex}
\end{minipage}
\end{figure*}

\section{Additional Discussion and Future Directions}
For future work, we'd like to push further on our results by scaling to larger ImageNet models to achieve state-of-the-art in test accuracy
alongside other metrics.
Although we focus on variational inference in this paper, applying this parameterization in MCMC is a promising parameter-efficient strategy for scalable BNNs.
As an alternative to using mixtures trained with the average per-component log-likelihood, one can use multiple independent chains over the rank-1 factors. Another direction for future work is the straightforward extension to higher rank factors. However, prior work \citep{swiatkowski2019ktied, izmailov2019subspace} has demonstrated diminishing returns that practically stop at ranks 3 or 5.

One surprising finding in our experimental results is that heavy-tailed priors, on a low-dimensional subspace, can significantly improve robustness and uncertainty calibration while maintaining or improving accuracy. This is likely due to the heavier tails allowing for more points in loss landscape valleys to be covered, whereas a mixture of lighter tails could place multiple modes that are nearly identical. 
However, with deeper or recurrent architectures, samples from the heavy-tailed posteriors seem to affect the stability of the training dynamics, leading to slightly worse predictive performance. One additional direction for future work is to explore ways to stabilize automatic differentiation through such approximate posteriors or to pair heavy-tailed priors with sub-Gaussian posteriors.

\section{Choices of Loss Functions}
\label{sec:nll}
\subsection{Definitions}
\begin{align*}
\begin{split}
\mathbf{x} &\in \mathbb{R}^d, \quad \mathbf{y}_c \in \{0, 1\}, \, \sum_{c=1}^C \mathbf{y}_c = 1 \\
\textbf{logits} &= f(\mathbf{x}, \boldsymbol{\theta}) \\
\textbf{probs} &= \operatorname{softmax}(\textbf{logits}) \\
\operatorname{softmax}(\boldsymbol{\lambda}) &= \frac{e^{\boldsymbol{\lambda}} }{\sum_{i=1}^{\|\boldsymbol{\lambda}\|} e^{\boldsymbol{\lambda}_i}} \\
p(\mathbf{y} | \mathbf{x}, \boldsymbol{\theta}) &= \operatorname{Categorical}(\mathbf{y}; \textbf{probs}) \\
&= \prod_{c=1}^C (\operatorname{softmax}(f(\mathbf{x}, \boldsymbol{\theta}))_c)^{\mathbf{y}_c} \\
- \log p(\mathbf{y} | \mathbf{x}, \boldsymbol{\theta}) &= -\sum_{c=1}^C \mathbf{y}_c \log \operatorname{softmax}(f(\mathbf{x}, \boldsymbol{\theta}))_c \\
&= - \mathbf{y}^{\top} \log \operatorname{softmax}(f(\mathbf{x}, \boldsymbol{\theta})) \\
M &= \textrm{num\_weight\_samples} \\
C &= \textrm{num\_classes} \\
\end{split}
\end{align*}

\subsection{Negative log-likelihood of marginalized \textit{logits}}
\begin{equation}
\begin{split}
&= - \mathbf{y}^{\top} \log \operatorname{softmax}\left(\int f(\mathbf{x}, \boldsymbol{\theta}) p(\boldsymbol{\theta}) d \boldsymbol{\theta}\right) \\
&\approx - \mathbf{y}^{\top} \log \operatorname{softmax}\left(\frac{1}{M} \sum_{m=1}^M f(\mathbf{x}, \boldsymbol{\theta}^{(m)})\right) \\
\end{split}
\label{eqn:marginal_logit_nll} 
\end{equation}

\subsection{Negative log-likelihood of marginalized \textit{probs}}
\begin{equation}
\begin{split}
&= - \mathbf{y}^{\top} \log \left\{ \int \operatorname{softmax}(f(\mathbf{x}, \boldsymbol{\theta})) p(\boldsymbol{\theta}) d \boldsymbol{\theta} \right\} \\
&\approx - \mathbf{y}^{\top} \log \left\{ \left( \frac{1}{M} \sum_{m=1}^M \operatorname{softmax}(f(\mathbf{x}, \boldsymbol{\theta}^{(m)})) \right) \right\} \\
\end{split}
\label{eqn:marginal_prob_nll} 
\end{equation}

\subsection{Marginal Negative log-likelihood (i.e., average NLL or Gibbs cross-entropy)}
\begin{equation}
\begin{split}
&= \mathbb{E}_{p(\boldsymbol{\theta})}[- \log p(\mathbf{y} | \mathbf{x}, \boldsymbol{\theta})] \\
&= \int - \log \left\{ p(\mathbf{y} | \mathbf{x}, \boldsymbol{\theta}) \right\} p(\boldsymbol{\theta}) d \boldsymbol{\theta}\\
&\approx \frac{1}{M} \sum_{m=1}^M \left\{ - \log p(\mathbf{y} | \mathbf{x}, \boldsymbol{\theta}^{(m)}) \right\} \\
\end{split}
\label{eqn:marginal_nll} 
\end{equation}

\subsection{Negative log marginal likelihood (i.e., mixture NLL)}
\begin{equation}
\begin{split}
&= - \log p(\mathbf{y} | \mathbf{x}) \\
&= - \log \left\{ \int p(\mathbf{y} | \mathbf{x}, \boldsymbol{\theta})  p(\boldsymbol{\theta}) d \boldsymbol{\theta} \right\} \\
&\approx - \log \left\{ \frac{1}{M} \sum_{m=1}^M p(\mathbf{y} | \mathbf{x}, \boldsymbol{\theta}^{(m)}) \right\} \\
&= - \log \left\{ \sum_{m=1}^M p(\mathbf{y} | \mathbf{x}, \boldsymbol{\theta}^{(m)}) \right\} + \log M
\\
&= - \log \left\{ \sum_{m=1}^M \exp{ \log p(\mathbf{y} | \mathbf{x}, \boldsymbol{\theta}^{(m)}) } \right\} + \log M
\\
&= - \operatornamewithlimits{logsumexp}_m \left\{ \log p(\mathbf{y} | \mathbf{x}, \boldsymbol{\theta}^{(m)}) \right\} + \log M \\
\end{split}
\label{eqn:mixture_nll} 
\end{equation}

As we saw in \Cref{sec:rank-1}, due to Jensen's inequality, (\ref{eqn:mixture_nll}) $\leq$ (\ref{eqn:marginal_nll}).
However, we find that minimizing the upper bound (i.e. Eq. \ref{eqn:marginal_nll}) to be easier while allowing for improved generalization performance.
Note that for classification problems (i.e., Bernoulli or Categorical predictive distributions), Eq. \ref{eqn:marginal_prob_nll} is equivalent to Eq. \ref{eqn:mixture_nll}, though more generally, marginalizing the parameters of the predictive distribution before computing the negative log likelihood (Eq. \ref{eqn:marginal_prob_nll}) is different from marginalizing the likelihood before taking the negative log (Eq. \ref{eqn:mixture_nll}), and from marginalizing the negative log likelihood (Eq. \ref{eqn:marginal_nll}). Also note that though they are mathematically equivalent for classification, the formulation of Eq. \ref{eqn:mixture_nll} is more numerically stable than Eq. \ref{eqn:marginal_prob_nll}.

\clearpage
\section{Out-of-distribution Performance}
\subsection{CIFAR-10-C Results}
\label{sec:cifar10-c}
\begin{figure}[h!]
    \centering
\begin{subfigure}[h]{\figwidth}
\includegraphics[width=\figwidth]{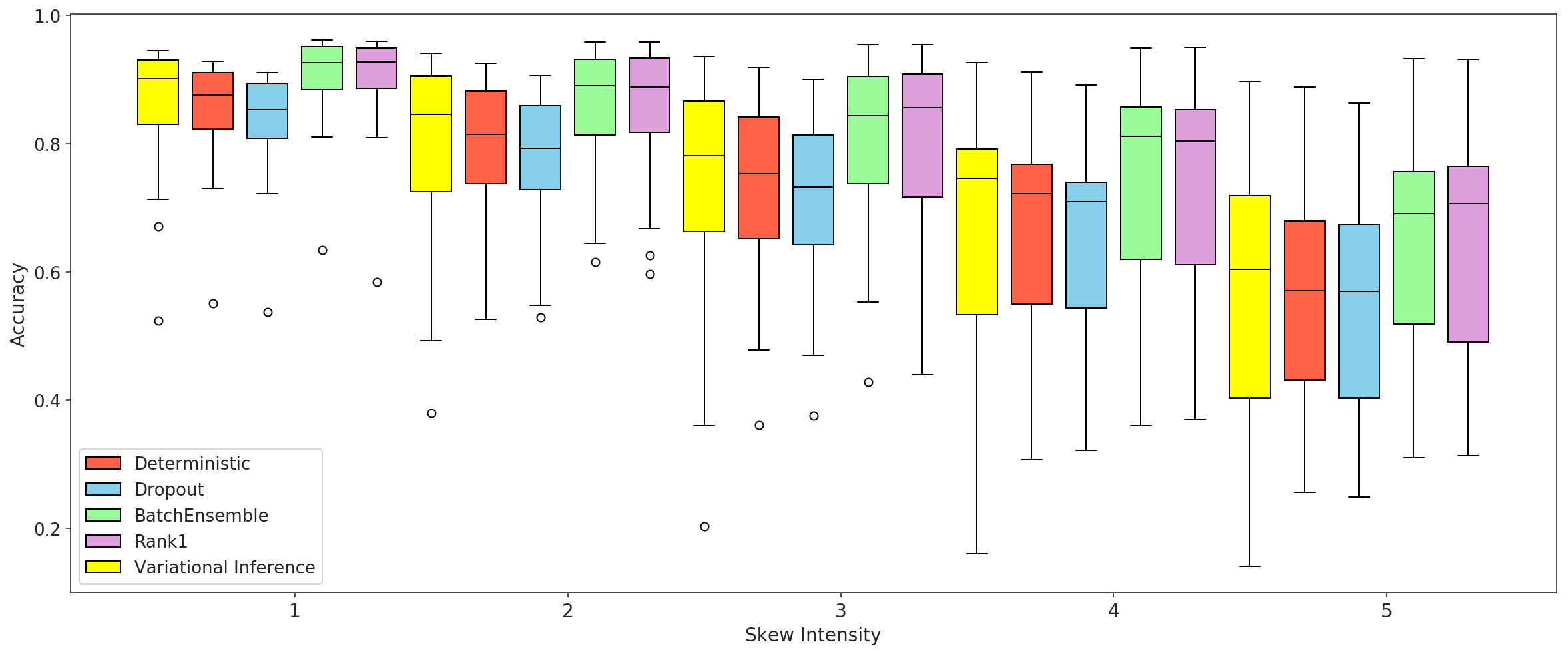}
\caption{Accuracy (higher is better).}
\end{subfigure} 
\begin{subfigure}[h]{\figwidth}
\includegraphics[width=\figwidth]{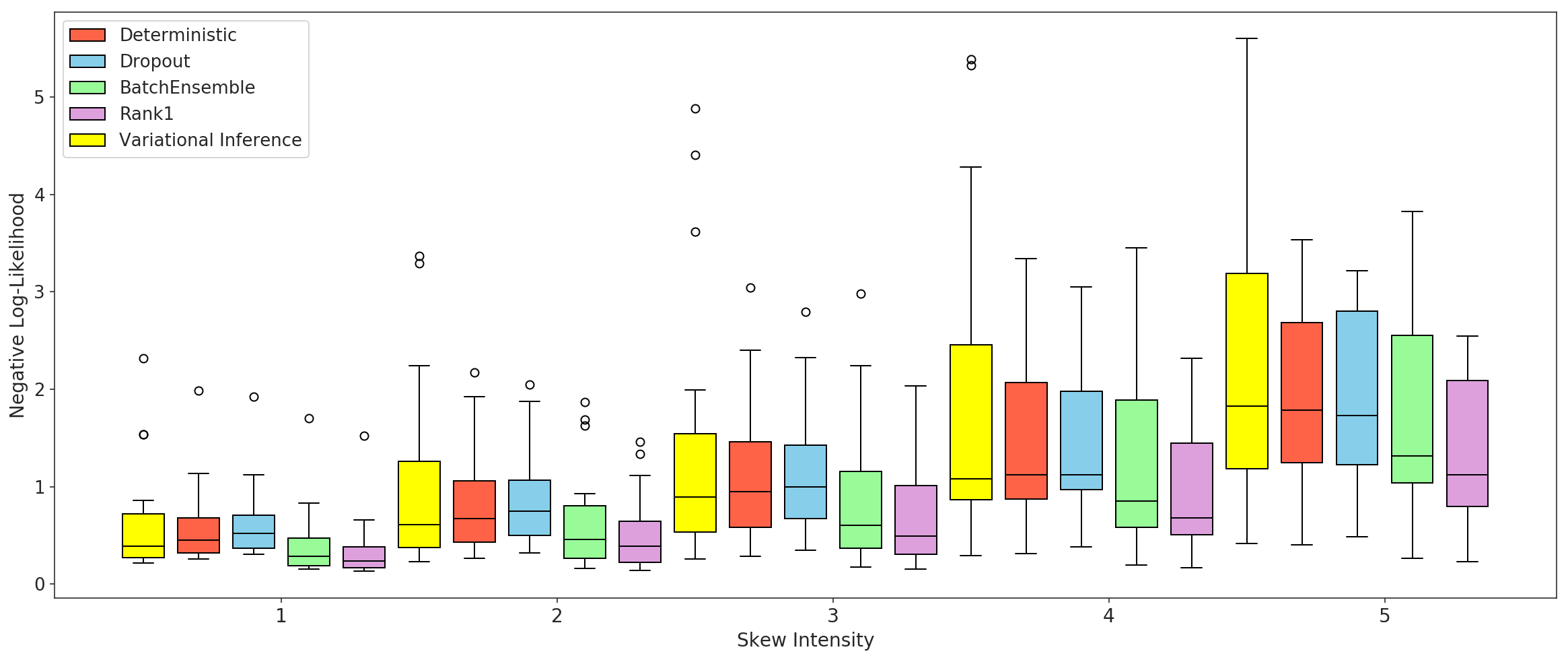}
\caption{Negative log-likelihood (lower is better).}
\end{subfigure} 
\begin{subfigure}[h]{\figwidth}
\includegraphics[width=\figwidth]{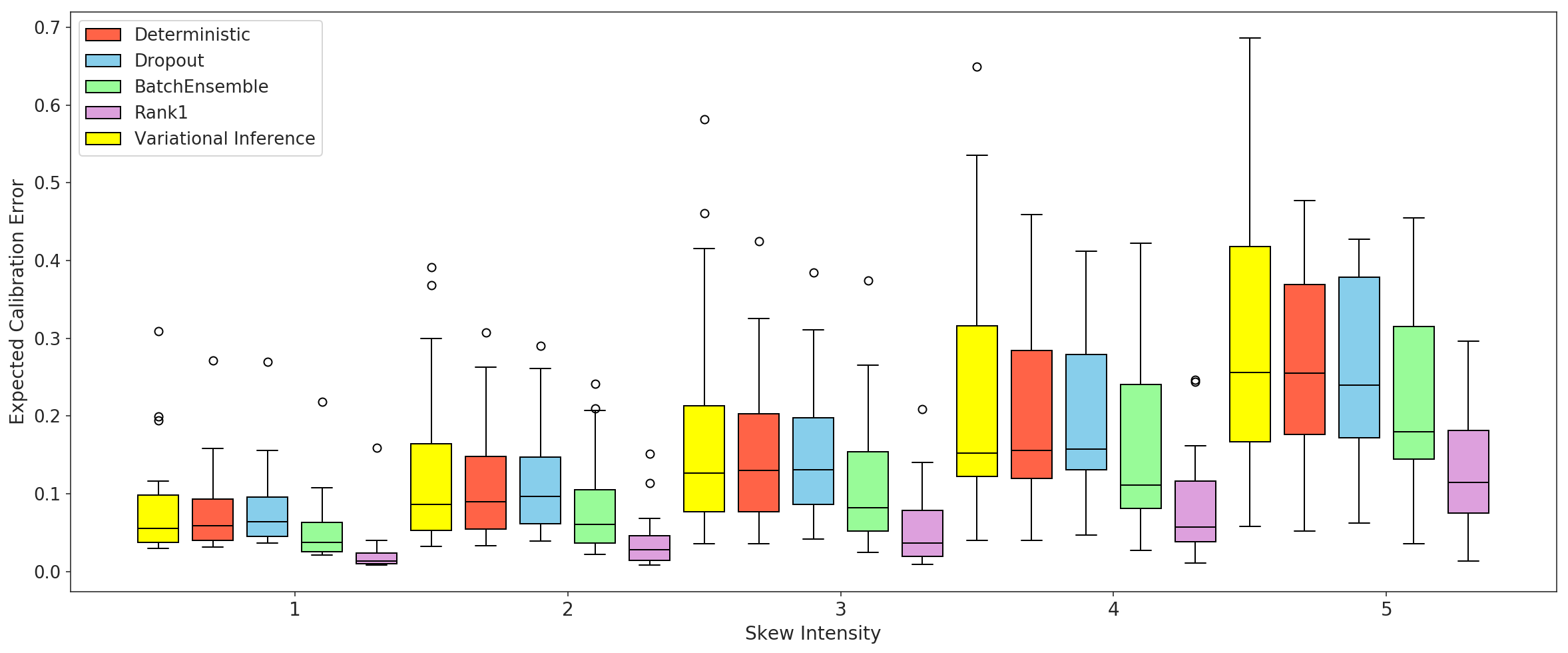}
\caption{ Expected calibration error  (lower is better).}
\end{subfigure} 
    \caption{Results on CIFAR-10-C showing median  performance across corruption types, and for increasing settings of the skew intensity. 
    }
    \label{fig:cifar10c-results}
\end{figure}

\clearpage

\subsection{CIFAR-100-C Results}
\label{sec:cifar100-c}

\begin{figure}[h!]
    \centering
\begin{subfigure}[h]{\figwidth}
\includegraphics[width=\figwidth]{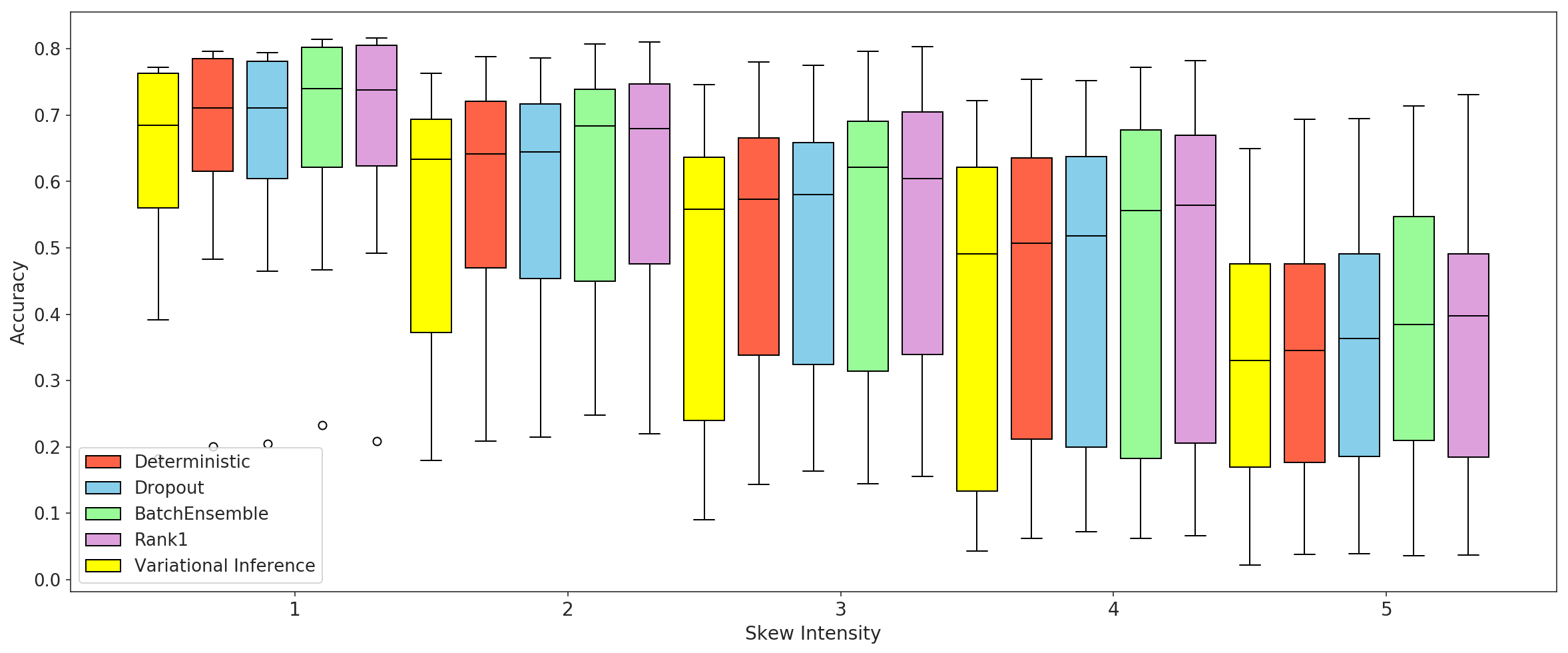}
\caption{Accuracy (higher is better).}
\end{subfigure} 
\begin{subfigure}[h]{\figwidth}
\includegraphics[width=\figwidth]{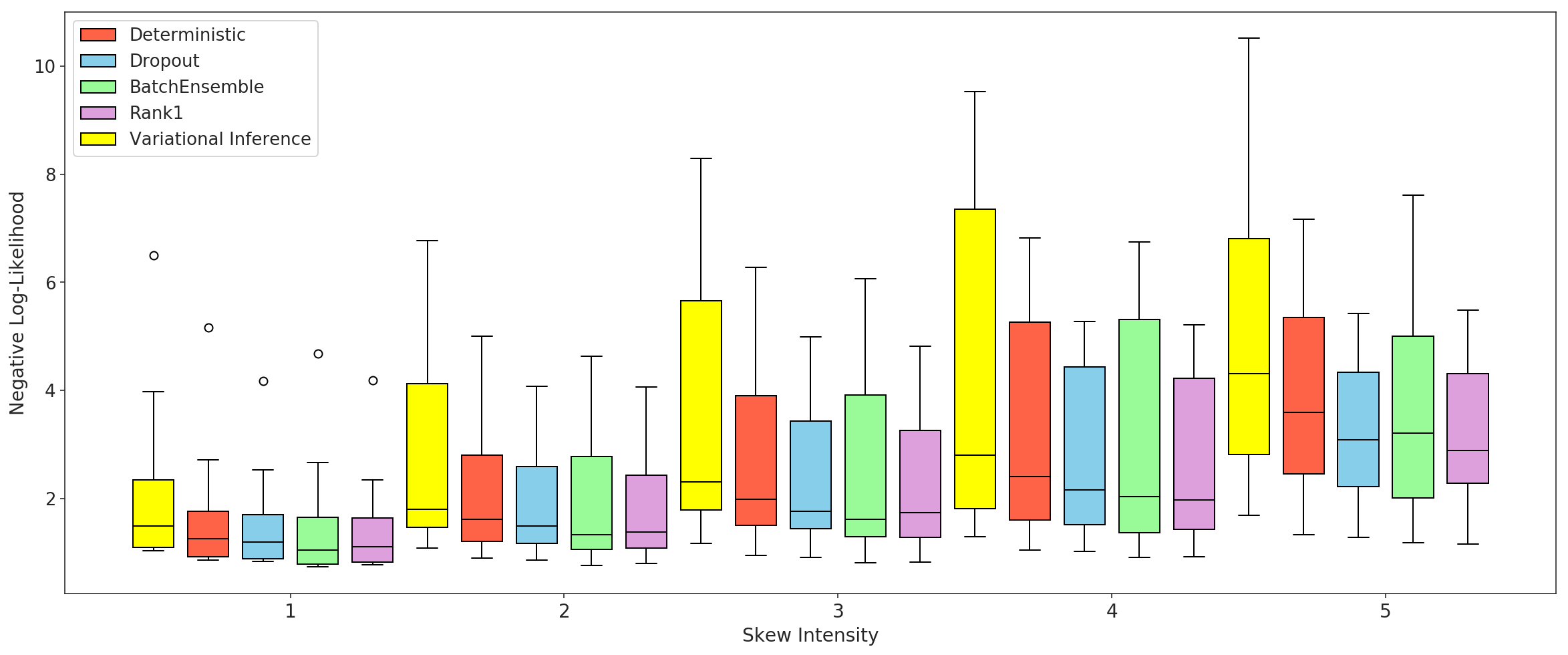}
\caption{Negative log-likelihood (lower is better).}
\end{subfigure} 
\begin{subfigure}[h]{\figwidth}
\includegraphics[width=\figwidth]{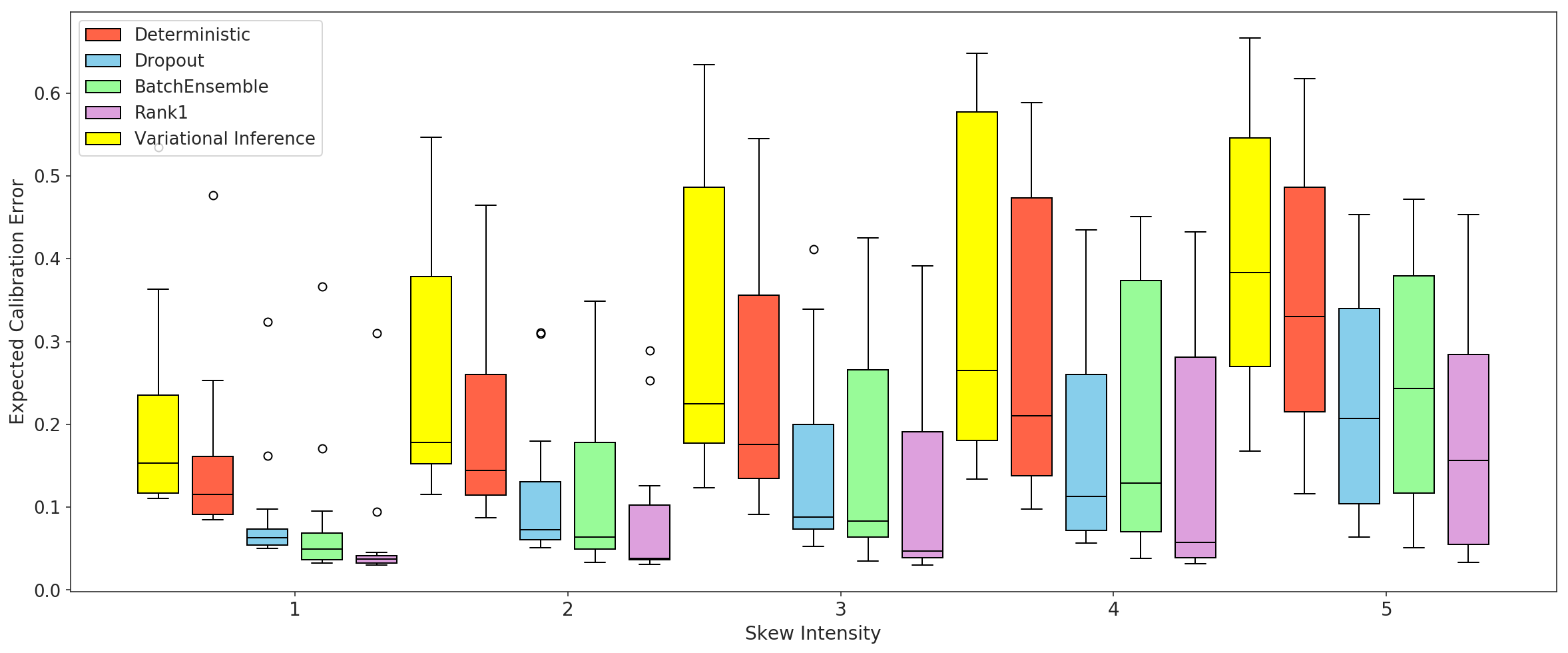}
\caption{ Expected calibration error  (lower is better).}
\end{subfigure} 
    \caption{Results on CIFAR-100-C showing median  performance across corruption types, and for increasing settings of the skew intensity. 
    }
    \label{fig:cifar100c-results}
\end{figure}

\clearpage
\subsection{ImageNet-C Results}
\label{sec:imagenet-c}

\begin{figure}[h!]
    \centering
\begin{subfigure}[h]{\figwidth}
\includegraphics[width=\figwidth]{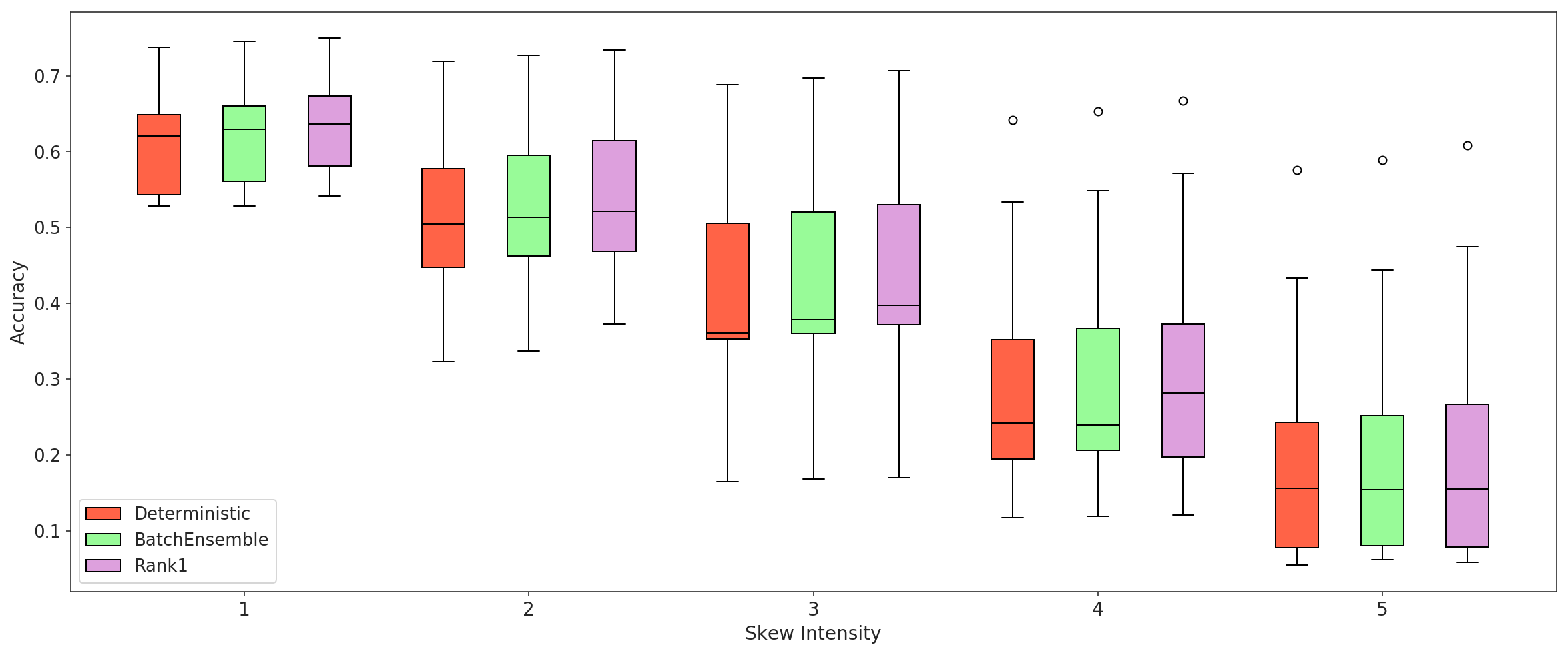}
\caption{Accuracy (higher is better).}
\end{subfigure} 
\begin{subfigure}[h]{\figwidth}
\includegraphics[width=\figwidth]{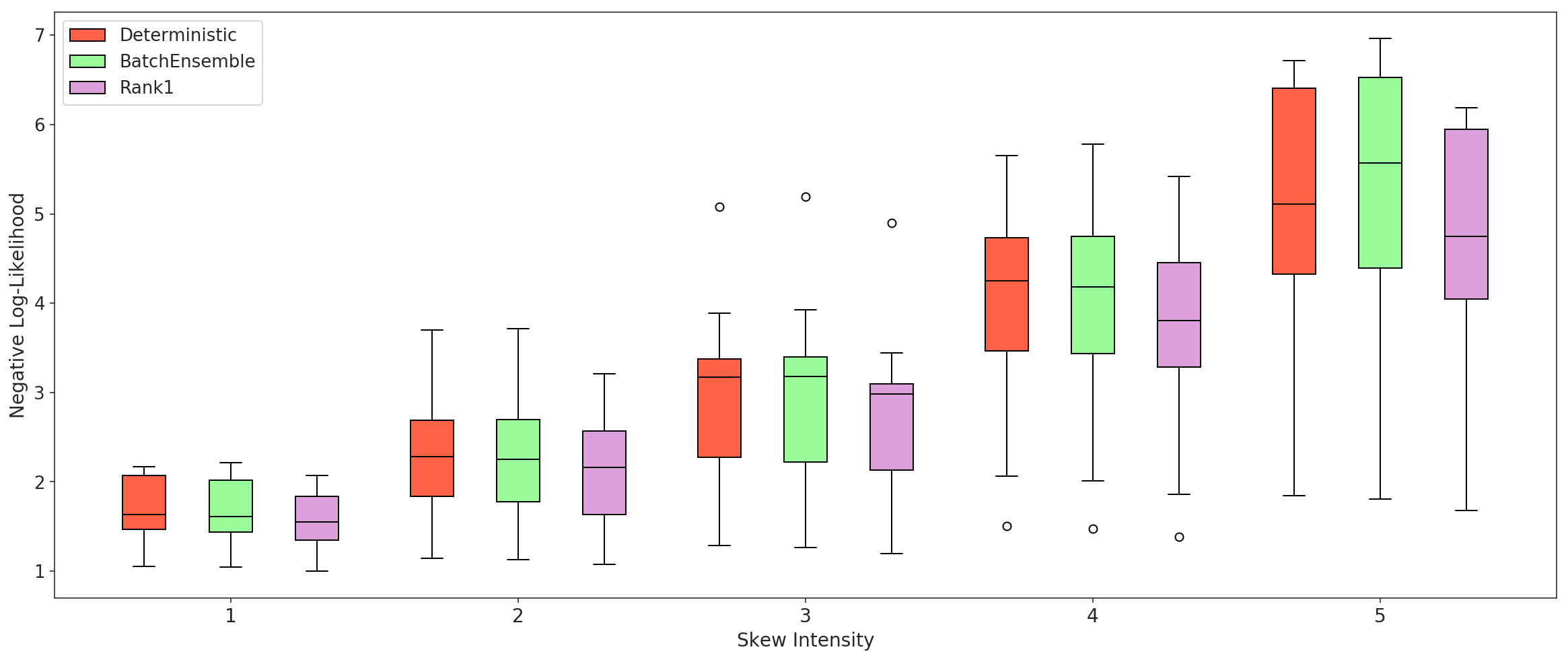}
\caption{Negative log-likelihood (lower is better).}
\end{subfigure} 
\begin{subfigure}[h]{\figwidth}
\includegraphics[width=\figwidth]{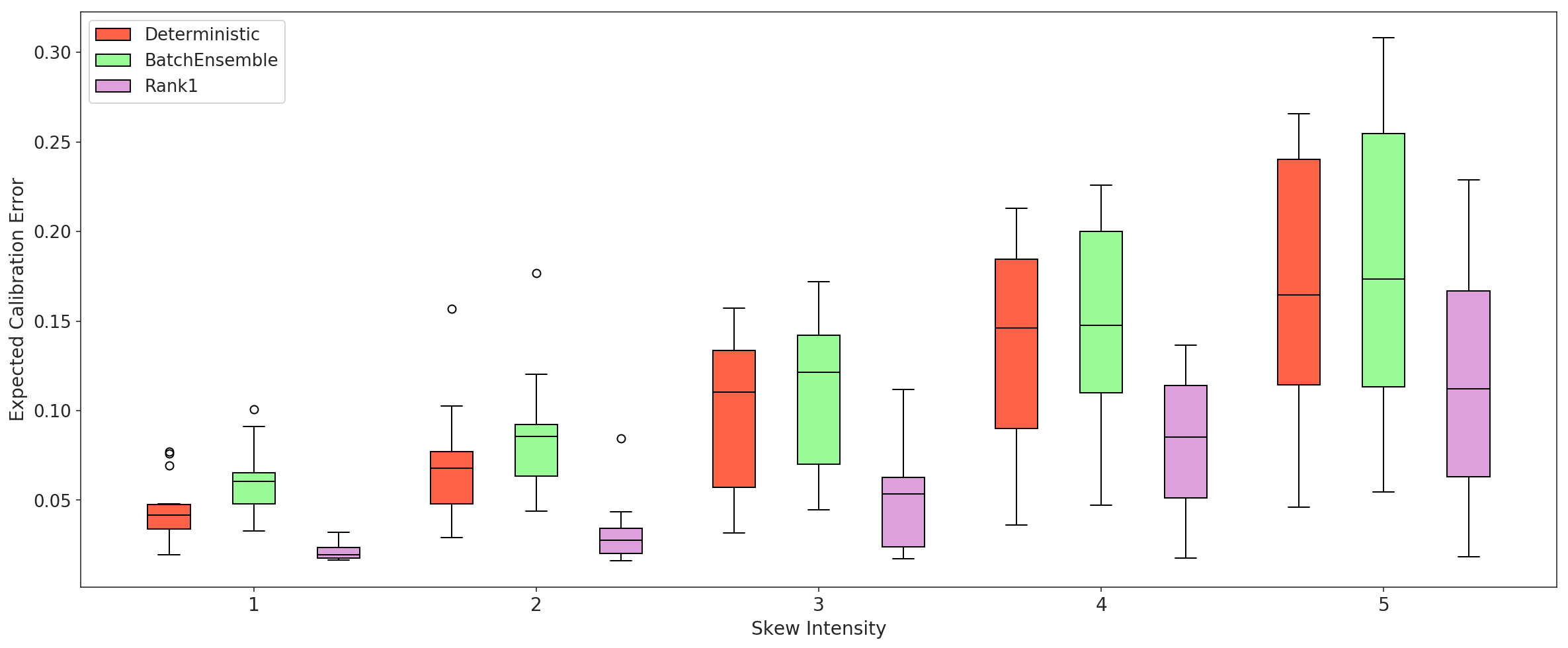}
\caption{ Expected calibration error  (lower is better).}
\end{subfigure} 
    \caption{Results on ImageNet-C showing median  performance across corruption types, and for increasing settings of the skew intensity. 
    }
    \label{fig:imagenet-c-results}
\end{figure}

\end{document}